\documentclass[nonblindrev]{informs3_noinforms}

\OneAndAHalfSpacedXI

\usepackage{bm}
\usepackage{cleveref}
\newcommand{\var}{\operatorname{Var}}
\newcommand{\E}{{\mathbb{E}}}
\usepackage{dsfont} 
\usepackage{algorithm}
\usepackage{algorithmicx}
\usepackage{algpseudocode} 

\DeclareMathOperator{\sd}{sd}

\usepackage{natbib}
 \bibpunct[, ]{(}{)}{,}{a}{}{,}%
 \def\bibfont{\small}%
 \def\newblock{\ }%

\makeatletter
\long\def\@makefigurecaption#1#2{\EGT\FigureCaptionFontStyle
  \begin{Center}\HD{9}{0}{\FigureNameFontStyle #1\kern16pt}\ignorespaces #2\HD{0}{0}\endgraf\end{Center}}
\makeatother

\TheoremsNumberedThrough     
\ECRepeatTheorems

\EquationsNumberedThrough    

\begin{document}


\RUNAUTHOR{Chen and Lu}

\RUNTITLE{Bandit Allocational Instability}

\TITLE{Bandit Allocational Instability}

\ARTICLEAUTHORS{%
\AUTHOR{Yilun Chen}
\AFF{The Chinese University of Hong Kong, Shenzhen, \EMAIL{chenyilun@cuhk.edu.cn}}
\AUTHOR{Jiaqi Lu}
\AFF{The Chinese University of Hong Kong, Shenzhen, \EMAIL{lujiaqi@cuhk.edu.cn}}
} 

\ABSTRACT{
    When multi-armed bandit (MAB) algorithms allocate pulls among competing arms, the resulting allocation can exhibit huge variation. This is particularly harmful in modern applications such as learning-enhanced platform operations and post-bandit statistical inference. Thus motivated, we introduce a new performance metric of MAB algorithms termed \emph{allocation variability}, which is the largest (over arms) standard deviation of an arm's number of pulls. We establish a fundamental trade-off between allocation variability and regret, the canonical performance metric of reward maximization. In particular, for any algorithm, the worst-case regret $\mathcal R_T$ and worst-case allocation variability $\mathcal S_T$ must satisfy $\mathcal R_T \cdot \mathcal S_T=\Omega(T^{\frac{3}{2}})$ as $T\rightarrow\infty$, as long as $\mathcal R_T=o(T)$. This indicates that any minimax regret-optimal algorithm must incur worst-case allocation variability $\Theta(T)$, the largest possible scale; while any algorithm with sublinear worst-case regret must necessarily incur $\mathcal S_T= \omega(\sqrt{T})$. We further show that this lower bound is essentially tight, and that any point on the Pareto frontier $\mathcal R_T \cdot \mathcal S_T=\tilde{\Theta}(T^{3/2})$ can be achieved by a simple tunable algorithm \texttt{UCB-f}, a generalization of the classic \texttt{UCB1}. Finally, we discuss implications for platform operations and for statistical inference, when bandit algorithms are used. As a byproduct of our result, we resolve an open question of \citet{praharaj2025instability}.
}%



\maketitle

%

\section{Introduction}\label{sec:intro}

Multi-armed bandit (MAB) is a canonical model of sequential allocation under uncertainty \citep{gittins2011multi}, with broad applications across Operations Research, Statistics, Economics, and Computer Science. In the classical formulation, a decision maker repeatedly selects one of $K$ arms. Pulling arm $i$ yields a random reward drawn independently from an arm-specific distribution that is unknown a priori. Over a horizon of $T$ pulls, the decision maker must balance \emph{exploration} (learning the reward distributions and identifying the most rewarding arm) and \emph{exploitation} (leveraging what has been learned) in order to maximize expected cumulative reward, or equivalently, minimize the expected regret relative to the best (most rewarding) arm. This exploration--exploitation dilemma has fueled decades of research and has led to a rich theory of regret-optimal algorithms.

While most bandit algorithms are designed to maximize reward, how do they allocate pulls?  {Somewhat surprisingly, numerical evidence shows that in certain cases, the allocation pattern of canonical bandit algorithms can be highly variable.}
We illustrate this in Example~\ref{example:intro} with \texttt{Thompson Sampling} (\cite{thompson1933likelihood}, \cite{russo2014learning}), a popular Bayesian heuristic enjoying multiple regret-optimal guarantees.

\begin{example}\label{example:intro}
    Consider a two-armed bandit instance with Bernoulli($p_i$) arms over a horizon of $T = 5000$. The arms are identical with $p_1 = p_2 = 0.5$. We simulate the behavior of \texttt{Thompson Sampling} (with $\textsf{Unif}[0,1]$ prior over $p$) across $20000$ independent trials. In each trial, we record the number of pulls of arm 1, and plot its empirical distribution across trials as the histogram in Figure~\ref{fig:example1}. Observe that the number of pulls allocated to arm 1 varies drastically across different trials, spreading essentially across the entire range from $0$ to $5000$, exhibiting the {maximal} possible variability\footnote{In recent studies of \cite{kalvit2021closer} {and \cite{halder2025stable}}, similar phenomena were observed under \texttt{Thompson Sampling} when both arms have Bernoulli(1) {or Normal(0,1)} rewards.}.  
\end{example}
\begin{figure}[hb]
  \centering
  \includegraphics[width=0.7\linewidth]{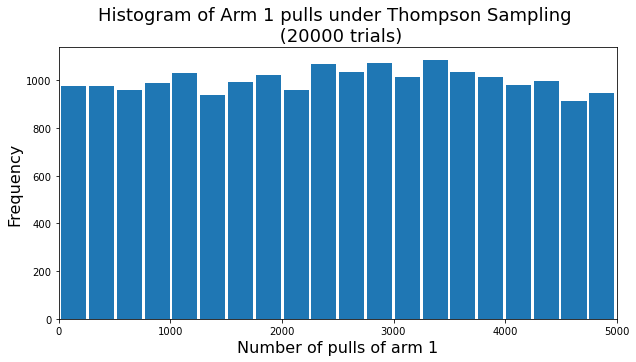}
  \caption{Allocation pattern of \texttt{Thompson Sampling} in Example~\ref{example:intro}.}
  \label{fig:example1}
\end{figure}
    We term the phenomenon in Example~\ref{example:intro} \textit{allocational instability}. Allocational instability highlights a distinct dimension of algorithmic behavior that is not captured by regret alone, and can be particularly harmful in modern bandit applications. For example:
    \begin{itemize}
        \item On a typical content-sharing platform, users arrive sequentially to browse content that are created by competing third-party creators. Upon each arrival, the platform chooses which contents to display or recommend. The platform often uses bandit algorithms to maximize user satisfaction. In the bandit abstraction, creators correspond to ``arms'', and ``pulling an arm'' corresponds to displaying the creator's content to an arriving user. User's satisfaction, often proxied by click-through rate and session length, serves as the reward. An arm's number of pulls represents the total traffic routed to the arm (creator) by the algorithm. When the algorithm exhibits allocational instability, creators subsequently face highly volatile traffic. This can have severe consequences such as perceived unfairness, increased churn risk, reduced effort and lower content quality, which will hurt the platform’s long-term prosperity.
\end{itemize} 
\indent\indent Similar negative consequences of algorithmic allocational instability also arise in e-commerce, service platforms and other settings where agents correspond to arms. Beyond those, allocational instability can also have undesirable consequences for post-bandit statistical inference (see \Cref{sec:literature_review} and \Cref{sec:implications} for details).

    Thus motivated, in this work we introduce a new performance metric, termed \emph{allocation variability}, that quantifies algorithmic allocational instability. In the stochastic MAB model, we define allocation variability as $S_T\triangleq \max_{i\in[K]} \sd\!\big(N_{i,T}\big),$ namely, the largest (over arms) standard deviation of the number of pulls by time $T$. We study allocation variability alongside the standard performance metric of expected pseudo-regret, $R_T\triangleq\sum_{i=1}^K \Delta_i\,\E\big[N_{i,T}\big],$ where $\Delta_i$ denotes the mean-reward gap between the best arm and arm $i$. Adopting a minimax viewpoint, we evaluate algorithms by the worst-case (instance-independent) metrics $\mathcal{S}_T \triangleq\sup_{\nu\in\mathfrak{P}} S_T(\nu)$ and $\mathcal{R}_T \triangleq \sup_{\nu\in\mathfrak{P}} R_T(\nu)$ over some instance class $\mathfrak{P}.$ {We ask:}\\

\emph{Can an algorithm simultaneously achieve small worst-case regret $\mathcal{R}_T$ and small worst-case allocation variability $\mathcal{S}_T$?}\\

As the main message of this paper, we show that one \textbf{cannot} simultaneously do well on both fronts in a fundamental sense. In particular, we establish a trade-off between $\mathcal{R}_T$ and $\mathcal{S}_T$, under which any reduction in regret necessarily comes at a commensurate price in allocation variability. We then provide a class of algorithms that attains the Pareto frontier up to polylog $T$ factors. Our main findings and contributions are summarized below.

\subsection{Contributions}
\paragraph{Fundamental regret-allocation-variability trade-off.} 
We establish a fundamental impossibility result (Theorem~\ref{thm:lower-bound}): over Gaussian bandit instances, any algorithm that \emph{actively learns} (i.e., achieves $\mathcal{R}_T=o(T)$) must satisfy
\[
\mathcal{R}_T\cdot\mathcal{S}_T=\Omega\!\left(T^{3/2}\right)
\qquad\text{as }T\to\infty.
\]
As an immediate corollary, minimax regret-optimal algorithms with $\mathcal{R}_T=\Theta(\sqrt{T})$ (cf.~\cite{lattimore2020bandit}) must have $\mathcal{S}_T=\Omega(T)$. Since the allocation variability is at most linear in $T$, this implies $\mathcal{S}_T=\Theta(T)$. In words, minimax regret-optimal learning necessarily results in the largest possible allocation variability. We further demonstrate that under minimax regret-optimal algorithms, such maximal allocation variability occurs on all instances in which $\Delta_i = O(1/\sqrt{T})$ for some arm $i$ (\Cref{thm:minimax-worst-case-interval}). Moving away from minimax optimality, our lower bound quantifies how reductions in allocation variability must be paid for by a commensurate increase in regret. Specifically, algorithms with $\mathcal S_T=\Theta(T^{1-\alpha})$ must incur $\mathcal R_T=\Omega(T^{\frac{1}{2}+\alpha})$, $\forall\alpha\in(0,\frac{1}{2}]$. At the opposite end of the spectrum, we identify what we term the \textit{price of learning}: any algorithm that has any intention to learn, namely, achieving sub-linear regret as $T$ grows, must incur worst-case allocation variability $ \mathcal{S}_T = \omega(\sqrt{T})$. This is in sharp contrast to algorithms that completely give up on learning, such as round-robin or other static allocation rules, which can in principle achieve $\mathcal{S}_T = 0.$ 

\paragraph{Pareto frontier and the \texttt{UCB-f} algorithms.} We show that the lower bound $\mathcal{R}_T\cdot\mathcal{S}_T=\Omega\!\left(T^{3/2}\right)$ tightly characterizes the achievable regret-allocation-variability region for active learning algorithms, as depicted in Figure~\ref{fig:pareto-frontier}.
\begin{figure}[h]
        \centering
        \includegraphics[width=0.5\linewidth]{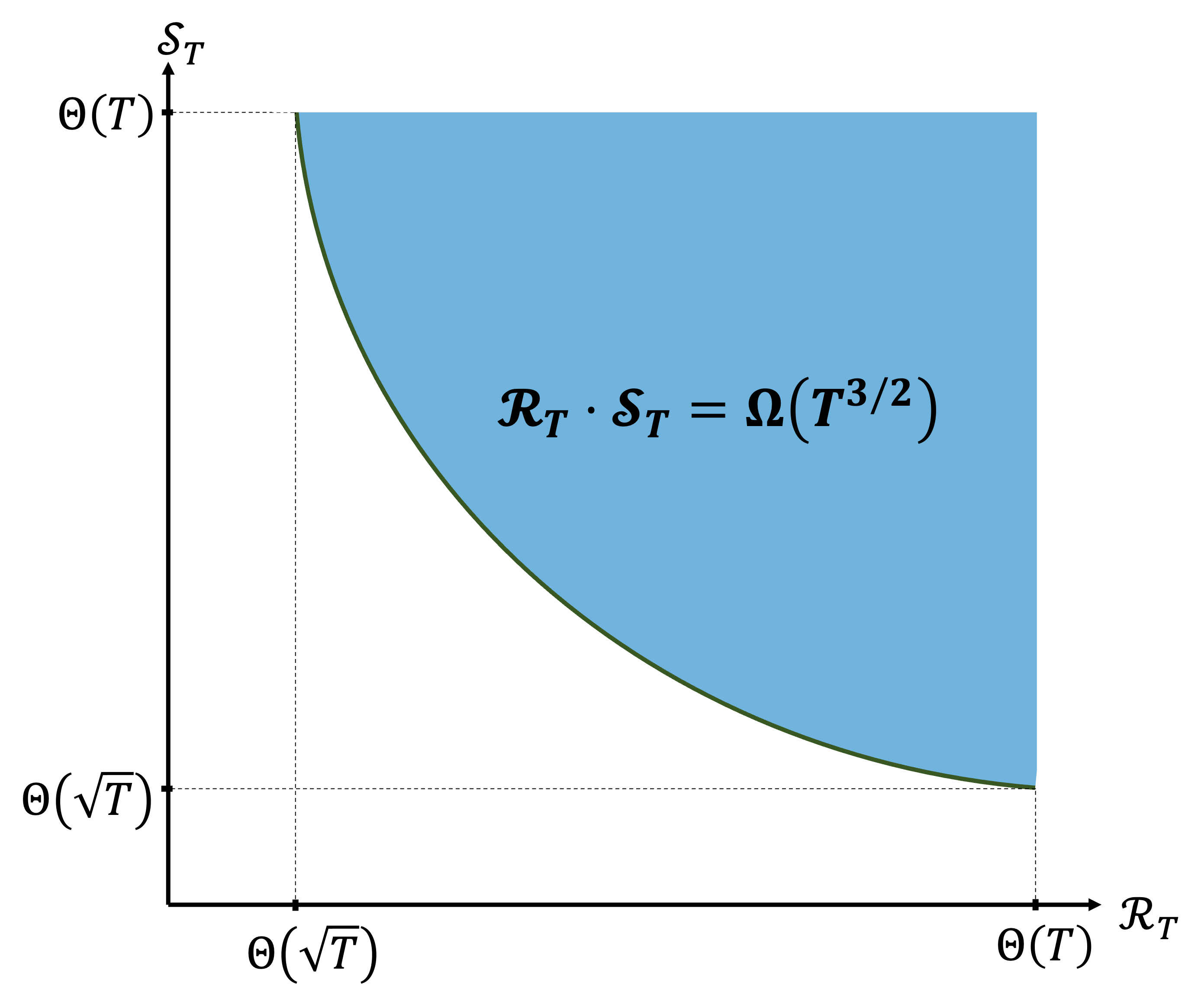}
        \caption{Pareto frontier of the two objectives of interest.}
        \label{fig:pareto-frontier}
    \end{figure}
In particular, we design a family of algorithms that can smoothly trade off regret and allocation variability along the boundary, up to polylog $T$ factors (\Cref{thm:ucb-f}).  This identifies $\mathcal{R}_T\cdot\mathcal{S}_T=\tilde{\Theta}(T^{3/2})$ as the Pareto frontier. Our Pareto-efficient family of algorithms, termed \texttt{UCB-f}, generalizes the canonical \texttt{UCB1} algorithm by replacing the standard exploration function (of order $\sqrt{\log t}$) with a tunable $f(t)$. A faster increasing $f(t)$ shifts the algorithm from aggressive exploitation towards encouraging more exploration, and correspondingly trades larger regret for reduced allocation variability.

\paragraph{Implications.} 
Our theoretical findings have rich implications. For learning-enhanced platform operations, we consider a platform that cares about both user satisfaction and agent (arm) utility, with minimax objective $\sup_{\nu\in\mathfrak{P}}(R_T(\nu)+(S_T(\nu))^\rho)$. The parameter $\rho\geq 0$ captures the agents' risk aversion level for demand variability. Our findings then immediately imply a lower bound of $(\sqrt{T}\vee (T^{\frac{3\rho}{2 + 2\rho}}\wedge T))$ for this objective, which is attained (up to polylog $T$ factors) by the \texttt{UCB-f} algorithm with an appropriate choice of $f(t)$ when $\rho\in[0,2]$. If $\rho>2$, i.e., the agents' risk-aversion level is so high, then the platform should completely give up learning and deterministically split demand among the agents to avoid the superlinear loss from allocation variability. For post-bandit statistical inference, we resolve an open question (with a negative answer) raised in \cite{praharaj2025instability} as a byproduct of Theorem~\ref{thm:lower-bound}--\ref{thm:minimax-worst-case-interval}. This implies that under minimax regret-optimal algorithms, the desired asymptotic normality of an arm's sample mean reward likely fails to hold, or at least can not be established by the techniques in the recent literature; see \Cref{sec:literature_review} and \Cref{sec:implications} for more details.

\paragraph{Novel lower bounding technique.}
To prove our main impossibility result, we develop a systematic approach that bridges regret and allocation variability, which may find a broader application. To convey the core intuition, consider a two-armed Gaussian bandit with rewards $\mathcal{N}(0,1)$ and $\mathcal{N}(-\Delta,1)$, parameterized by the gap $\Delta\geq 0$. Let $g_T(\Delta) = \mathbb{E}[N_{2, T}]$ denote the expected pull counts of the inferior arm. Then the expected pseudo-regret $R_T(\Delta) = \Delta g_T(\Delta)$. A key observation is that the allocation variability can be related to $g'_T$, namely, how fast $g_T(\cdot)$ changes with respect to $\Delta$. Precisely, let $\Delta'$ denote $\Delta+1/\sqrt{g_T(\Delta)}$, then $S_T(\Delta) \asymp \left|g_T(\Delta) - g_T\left(\Delta'\right)\right|$ by the following intuition. Under $\Delta$, there are on average $g_T(\Delta)$ pulls on the inferior arm and consequently, one can not effectively distinguish instances in $\left(\Delta, \Delta'\right)$ due to the noise in the samples. Thus under $\Delta$, the pull counts must ``stretch'' at times to behave as if the instance were $\Delta'$, leading to the unavoidable variability on the order of $\left|g_T(\Delta) - g_T\left(\Delta'\right)\right|$. Such an observation allows us to relate $S_T(\Delta)$ to ${g'_T(\Delta)}/{\sqrt{g_T(\Delta)}}$. The fundamental tension between controlling regret and reducing allocation variability then follows from the following reasoning. At $\Delta = 0, g_T(0) = \frac{T}{2}$ by symmetry. Starting from $g_T(0)$, a small regret requires a fast descent of $g_T(\Delta)$ as $\Delta$ increases, while a small allocation variability requires $g_T(\Delta)$ to change slowly (so $|g'_T(\Delta)|$ stays in control). Explicitly solving this tension results in the desired trade-off.

\subsection{Literature review}\label{sec:literature_review}
\paragraph{Regret-minimization in MAB.}
In the extensive stochastic MAB literature, the primary performance criterion for an algorithm has been the expected (pseudo-)regret, defined as the expected loss in cumulative reward relative to always pulling the optimal arm. For a fixed problem instance—unknown to the algorithm—the seminal work of \cite{lai1985asymptotically} establishes that the minimal achievable expected regret must scale as $\Theta(\log T)$ as $T\rightarrow\infty$. Complementary to this instance-dependent notion is the instance-independent, or worst-case, expected (pseudo-)regret (\cite{lattimore2020bandit}). It is well known that for stochastic MAB, the worst-case expected regret under any algorithm is lower bounded by $\Theta(\sqrt{T})$; see \cite{auer2002finite}. Several classic algorithms attain near-optimal rates under both instance-dependent and worst-case criteria, such as \texttt{UCB1} (\cite{auer2002finite}), \texttt{UCB-V} (\cite{audibert2009exploration}), \texttt{KL-UCB} (\cite{garivier2011kl}) and \texttt{Thompson Sampling} (\cite{agrawal2012analysis, daniel2018tutorial}). There are also various algorithms designed to achieve strict optimality. For example, in terms of worst-case optimal rate $\Theta(\sqrt{T})$, there are \texttt{KL-UCB++} (\cite{menard2017minimax}), \texttt{MOSS} and \texttt{Anytime-MOSS} (\cite{audibert2009minimax}), \texttt{Vanilla-MOSS} (\cite{lattimore2020bandit}), \texttt{ADA-UCB} (\cite{lattimore2018refining}), \texttt{OC-UCB} (\cite{lattimore2015optimally}), \texttt{KL-MOSS} (\cite{praharaj2025instability}), \texttt{KL-UCB-SWITCH} and \texttt{Anytime KL-UCB-SWITCH} (\cite{garivier2022kl}), \texttt{MOTS} (\cite{jin2021mots}), \texttt{DETC} (\cite{jin2021double}), \texttt{Tsallis-INF} (\cite{zimmert2021tsallis}), and more. Our result highlights a fundamental side effect of allocational instability for all these algorithms.

\paragraph{Beyond-expectation analysis of regret.}  A recent stream of work studies bandit performance beyond expected regret. In particular, \cite{fan2021fragility,simchi2022simple,simchi2023multi} investigate the regret tail probability $\mathbb{P}(\hat{R}_T > a )$ for large $a \gg \mathbb{E}[\hat{R}_T]$, characterizing when it decays slowly, namely, exhibiting a heavy tail. In the textbook \cite{lattimore2020bandit}, this phenomenon is discussed informally for the MOSS algorithm. \cite{fan2021fragility} show that, in general, the regret tails of instance-dependent–optimal algorithms (achieving the Lai–Robbins lower bound) must be heavy-tailed, and can be as heavy as a truncated Cauchy (of rate $\Theta(\frac{1}{T})$) for certain distribution classes. Regret tail risk differs from allocation variability in several fundamental ways.  In particular, severe power-law regret tails only arise on instances with a clearly suboptimal arm (with gap $\Delta_i > 0$ independent of $T$), and can be effectively eliminated with only a modest sacrifice in expected regret (e.g., by a constant or polylog factor), using tunable algorithms such as \texttt{UCB-f} \cite{fan2021fragility, simchi2023regret}. In contrast, allocation variability is maximized when arms are indistinguishable, and it's far more stubborn and costly to mitigate: in general, one cannot achieve small allocation variability without a substantial sacrifice on the regret front. Beyond regret tail risk, \cite{fan2022typical} characterize the LLN and CLT for regret under \texttt{UCB1} and \texttt{Thompson Sampling}, while \cite{kalvit2021closer,kuang2024weak} further study the corresponding pathwise dynamics under these algorithms.

\paragraph{Post-bandit statistical inference.}
Recently, a stream of literature aims to understand whether the arm's sample mean reward computed from post-bandit data exhibits asymptotic normality, so that classic CLT-based inference methods can be applied. A useful technical lens to study this question is whether the bandit algorithm satisfies sampling stability, which is defined as $\frac{N_{i,T}}{n_{i,T}}\xrightarrow{p}1$ for some sequence of non-random values $n_{i,T}$, as $T\rightarrow\infty$ (see \cite{lai1982least} and \cite{praharaj2025instability}). Sampling stability is a key sufficient condition that guarantees the asymptotic normality of the sample mean, and it is recently frequently used to analyze existing algorithms and design new ones.  Along this line, \cite{kalvit2021closer} show the sampling stability of \texttt{UCB1}, and \cite{chen2025characterization} characterize its precise convergence rate.  \cite{han2026thompson} prove the negative for \texttt{Thompson Sampling}. Subsequently, \cite{khamaru2024inference} establish the sample mean reward's asymptotic normality for \texttt{UCB1}, and \cite{halder2025stable} propose ways to tune and stabilize \texttt{Thompson Sampling} while only incurring an additional logarithm factor in worst-case regret.  \cite{praharaj2025instability} prove that sampling stability does not hold for minimax regret-optimal \texttt{MOSS}-family algorithms. Until now, the question of whether simultaneously stable and minimax regret-optimal algorithms exist remains open (see \cite{praharaj2025instability}). Our fundamental lower bound in \Cref{thm:lower-bound} implies a negative answer to this question. Therefore the desired asymptotic normality of an arm's sample mean reward likely fails to hold, or at least can not be established by the techniques in the recent literature.

\paragraph{Economic aspects of learning algorithms.}
More broadly, this paper contributes to an emerging literature on the economic aspects of learning algorithms; see \cite{hartline2026ec} for an overview. This literature aims to understand the (equilibrium) outcomes of algorithmic actors as they are increasingly replacing human actors in the economy (\cite{foster1997calibrated}, \cite{hart2000simple}, \cite{calvano2020artificial}). A central theme is that regret minimization serves as a tractable notion of rationality for algorithmic decision makers. Our work complements this line of research by focusing on an intrinsic property of regret-minimizing algorithms, which we term allocation variability. We show that achieving low regret must entail substantial fluctuations in induced allocations, a feature that is largely absent from existing economic analyses of no-regret dynamics. Such allocational instability is economically relevant in platform operations and market design, where fluctuations in demand or traffic allocation may trigger strategic responses by participating agents (e.g. \cite{lee1997bullwhip}, \cite{cachon2017role}). In this sense, our results provide a new angle for understanding learning algorithms in economic systems.\\
\subsection{Organization of the paper}
The remainder of the paper is organized as follows. {\Cref{sec:preliminary} sets up the MAB model and the performance metric of allocation variability. \Cref{sec:lower-bound} establishes the fundamental trade-off between regret and allocation variability. We then introduce a family of \texttt{UCB-f} algorithms in \Cref{sec:algorithm} that can achieve any point on the Pareto frontier. \Cref{sec:implications} discusses the implications of our results. We prove \Cref{thm:lower-bound} in \Cref{sec:lower-bound-proof0}. \Cref{sec:conclusion} concludes the paper. The proofs of all other results are in the appendix.}

\section{Preliminary}\label{sec:preliminary}
Consider the classic stochastic MAB model. There are $K$ arms. For each arm $i\in[K]$, let $\{X^i_s\}_{s\ge 1}$ be an i.i.d.\ reward sequence with common distribution $\nu_i$. These sequences are mutually independent across arms. We refer to $\nu=(\nu_1,\dots,\nu_K)$ as a \textit{bandit instance}. Over a horizon $T$, a bandit algorithm selects a (potentially random) arm $A_t \in [K]$ at each time $t \in [T]$. Define $N_{i,t}\triangleq \sum_{s=1}^t \mathds{1}\{A_s=i\}$ as the pull counts (number of pulls) of arm $i$ up to (and include) time $t$, and $r_t \triangleq X^{A_t}_{N_{A_t,t}}$ as the reward collected at time $t$. Let $\mathcal{H}_{t-1}\triangleq (A_1,r_1,\ldots,A_{t-1},r_{t-1})$ be the observed history up to time $t-1$. The algorithm is specified by a sequence of decision rules $\{\pi_t\}_{t\in[T]}$, where each $\pi_t(\cdot\mid \mathcal{H}_{t-1})$
is a distribution over $[K]$, and the action is drawn as $A_t \sim \pi_t(\cdot\mid \mathcal{H}_{t-1}).$

The standard performance metric of a bandit algorithm is the regret. In particular, let $\mu_i$ be the mean reward of arm $i$ and $\mu^*\triangleq \max_{i\in[K]}\mu_i$, and define the mean reward gap $\Delta_i\triangleq \mu^*-\mu_i$. Then the (expected pseudo-)regret is given by $R_T\triangleq\sum_{i=1}^K \Delta_i\,\E[N_{i,T}]$. Minimizing $R_T$ is equivalent to maximizing the expected cumulative reward $\mathbb{E}\left[\sum_{t= 1}^T r_t \right].$ 

In addition to the regret, this work introduces a new performance metric that quantifies the variability of an algorithm's pull allocation. Specifically, we define the \emph{allocation variability} as the largest standard deviation of the pull counts $N_{i,T}$ across arms.
\begin{definition}[Allocation Variability] 
An algorithm's \emph{allocation variability} by time $T$ is defined as $S_T\triangleq \max_{i\in[K]} \sd\!\big(N_{i,T}\big) = \max_{i\in[K]} \sqrt{\var\big(N_{i,T}\big)}.$
\end{definition}   

We adopt the minimax framework, evaluating a bandit algorithm by its worst-case performance over a prescribed instance class. Formally, for an instance class $\mathfrak{P}$, define $\mathcal{R}_T(\mathfrak{P})\triangleq \sup_{\nu\in\mathfrak{P}} R_T(\nu)$ and $\mathcal{S}_T(\mathfrak{P})\triangleq \sup_{\nu\in\mathfrak{P}} S_T(\nu)$ as the worst-case (instance-independent) regret and allocation variability, respectively. Throughout, we focus on a sub-Gaussian instance class with bounded means, formalized below.
\begin{definition}[Sub-Gaussian instance class]\label{def:instance-class}
Fix $M, \sigma, \bar{K} > 0$. Let $\mathfrak{P}_{\mathrm{sg}}(M,\sigma, \bar{K})$ denote the class of stochastic bandit instances
$\nu=(\nu_1,\dots,\nu_K)$ with number of arms $2 \le K \le \bar{K}$, such that for each arm $i\in[K]$, the reward distribution $\nu_i$ has mean $\mu_i\in[-M,M]$ and is $\sigma$-sub-Gaussian.
\end{definition}
For brevity, we omit explicit dependence on $M, \sigma$ and $\bar{K}$, and write $\mathfrak{P}_{\mathrm{sg}}$ for the instance class, with the understanding that they are fixed constants and are typically immaterial to our $T$-asymptotic statements. We also suppress the dependence on $\mathfrak{P}_{\mathrm{sg}}$ and write $\mathcal{R}_T$ and $\mathcal{S}_T$ for notational simplicity.

We specify the class of algorithms considered. We first impose two mild regularity conditions: (i) if $\nu_i=\nu_j$ for some $i\neq j$, then $\E[N_{i,T}] = \E[N_{j,T}]$; and (ii) each arm is in expectation pulled at least once (for $T\ge K$). Both of the properties are satisfied by essentially all standard bandit algorithms of interest. Other than that, we impose the following  \emph{active-learning} condition that requires algorithms to have an intention to learn.

\begin{definition}[Active-learning algorithms]\label{def:active-learning}
A bandit algorithm is {active-learning} (over $\mathfrak{P}_{\mathrm{sg}}$) if $\lim_{T \to\infty}  \frac{\mathcal R_T}{T}\ =\ 0$, i.e., it achieves sublinear worst-case regret.
\end{definition}
{\subsection{Additional notation}
To simplify exposition, we sometimes use the classic Bachmann-Landau notation. 
{In particular, for two nonnegative sequences $\{a_T\}$ and $\{b_T\}$, we write $a_T = O(b_T)$ if there exist constants $C>0$ and $T_0$ such that $a_T \le C b_T$ for all $T \ge T_0$, and $a_T = \Omega(b_T)$ if $b_T = O(a_T)$. We write $a_T = \Theta(b_T)$ if $a_T = O(b_T)$ and $a_T = \Omega(b_T)$. Moreover, $a_T = o(b_T)$ means $a_T/b_T \to 0$ as $T\to\infty$, and $a_T = \omega(b_T)$ means $a_T/b_T \to \infty$ as $T\to\infty$. We also use $\tilde{O}(\cdot)$, $\tilde{\Omega}(\cdot)$, and $\tilde{\Theta}(\cdot)$ to suppress factors that are polylogarithmic in $T$.} For any positive integer $M$, $[M]$ denotes $\{1, 2, \dots, M\}$, the set of positive integers up to $M.$ For a random variable $X$, $\sd(X)$ and $\var(X)$ denote its standard deviation and variance, respectively. 
}

\section{{An Impossibility Theorem}}\label{sec:lower-bound}
{
We provide an impossibility theorem that characterizes the fundamental trade-off between learning efficiency (as measured by regret) and allocational stability (as measured by allocation variability).
\begin{theorem}\label{thm:lower-bound}
Over the instance class $\mathfrak{P}_{\mathrm{sg}}$, any active-learning algorithm must satisfy
\[
\liminf_{T\to\infty}\frac{\mathcal{R}_T\cdot\mathcal{S}_T}{T^{3/2}} \;>\; C,
\]
where $C>0$ is a constant that depends only on $\mathfrak{P}_{\mathrm{sg}}$ and is independent of the algorithm.
\end{theorem}
We prove \Cref{thm:lower-bound} in \Cref{sec:lower-bound-proof0}. A few remarks are in order below.
\paragraph{On the role of active-learning.} The active-learning condition is crucial. Indeed, an algorithm that does not actively learn can trivially achieve $\mathcal{S}_T=0$, and violates the lower bound in  Theorem~\ref{thm:lower-bound}. Consider, for instance, the round-robin rule (see more discussion in Section~\ref{sec: closer look}). 

The lower bound in Theorem~\ref{thm:lower-bound} is asymptotic in the horizon $T$. Improving it to a \textit{finite-time} bound necessarily requires a stronger notion of active-learning. To see this, consider a sequence of algorithms $\{\pi^{(l)}\}_{l\ge 1}$ where $\pi^{(l)}$ plays round-robin up to time $l$ and then switches to an arbitrary bandit algorithm e.g. \texttt{UCB1} from time $l+1$ onward. Each $\pi^{(l)}$ satisfies Definition~\ref{def:active-learning} since it eventually learns and hence has sublinear worst-case regret as $T\to\infty$. However, for any fixed horizon $T$, choosing $l=T$ yields a policy $\pi^{(T)}$ with $\mathcal{S}_T=0$. Consequently, no nontrivial lower bound of the form in Theorem~\ref{thm:lower-bound} can hold uniformly at every finite $T$ unless we strengthen active learning to enforce a uniform learning rate (thereby excluding such ``delayed-learning'' concatenations). We do not pursue this direction in the present work. 

\paragraph{On $K$-dependence.} 
Our impossibility result in Theorem~\ref{thm:lower-bound} is asymptotic in $T$ and it does not capture the dependence on the number of arms $K$ in the spirit of the classic minimax bandit literature. Indeed, we prove our lower bound using a special family of two-armed Gaussian bandit instances with bounded mean and unit variance (see also Section~\ref{sec: worst case} and Section~\ref{sec:lower-bound-proof0} for a detailed description). We leave a sharp characterization of the $K$-dependence to future work.

\subsection{A closer look at different regret-scaling regimes}\label{sec: closer look}
Theorem~\ref{thm:lower-bound} imposes a lower bound on the worst allocation variability, asserting that any active-learning algorithm must satisfy $\mathcal{S}_T = \Omega(\frac{T^{3/2}}{\mathcal{R}_T})$. We next zoom in on specific regret-scaling regimes and discuss the implications of Theorem~\ref{thm:lower-bound}.
\paragraph{The (near)minimax regret-optimal regime: $\mathcal{R}_T = \tilde{O}(\sqrt{T})$.} It is well known that the worst-case regret is lower bounded by $\mathcal{R}_T=\Omega(\sqrt{T})$ \cite{lattimore2020bandit}, and much of the bandit literature focuses on the design and analysis of algorithms achieving (near)minimax-optimal regret $\mathcal{R}_T=\tilde{O}(\sqrt{T})$. A striking implication of Theorem~\ref{thm:lower-bound} is that any such near minimax regret-optimal algorithm must incur nearly maximal allocation variability.
\begin{corollary}\label{cor:tradeoff-minimax-algorithms}
    Any algorithm with worst-case regret $\tilde{O}(\sqrt T)$ satisfies $\mathcal{S}_T = \tilde{\Theta}(T)$. In particular, minimax regret-optimal algorithms that attain $\mathcal{R}_T = \Theta(\sqrt{T})$ suffer a linear-in-$T$ worst-case allocation variability.
\end{corollary}
Corollary~\ref{cor:tradeoff-minimax-algorithms} follows from Theorem~\ref{thm:lower-bound} and the fact that $N_{i,T}$ is bounded by $T$, which immediately implies a linear-in-$T$ upper bound on its standard deviation. Many existing algorithms are designed for achieving the minimax-optimal regret $\mathcal{R}_T = \Theta(\sqrt{T})$ {(see \Cref{sec:literature_review} for a non-exhaustive list of such algorithms)}. Corollary~\ref{cor:tradeoff-minimax-algorithms} shows that these algorithms all must incur $\Theta(T)$ worst-case allocation variability.

Beyond exact minimax regret optimality, many popular algorithms, including \texttt{Thompson Sampling} \citep{agrawal2012analysis} and \texttt{UCB1} \citep{auer2002finite}, are well known to attain the near-minimax rate $\mathcal{R}_T = O(\sqrt{T\log T})$. For such algorithms, Theorem~\ref{thm:lower-bound} implies that, in the worst case, the pull count $N_{i,T}$ must fluctuate on the order of at least $T/\sqrt{\log T}$.  Relatedly, \citep{kalvit2021closer} study the sampling behavior of \texttt{UCB1} and show that $\frac{N_{i,T}}{T} \xrightarrow{p} 1$. Our theorem complements and, in some sense contrasts with this result, suggesting that the convergence happens at an extremely slow rate of $\Omega(1/\sqrt{\log T})$ in the worst-case. \citet{chen2025characterization} shows that this rate is tight. As for \texttt{Thompson Sampling}, the worst-case allocation variability is even larger than $T/\sqrt{\log T}$. Indeed, \cite{kalvit2021closer} and \cite{han2026thompson} show that its sampling rate $\frac{N_{i,T}}{T}$ does not converge to a constant when the optimal arm is non-unique. 

\paragraph{The weak-learning regime: $\mathcal{R}_T$ is nearly linear.} At the opposite end of the spectrum, we consider algorithms that barely satisfy active learning, in the sense that their worst-case regret is almost linear, e.g. $\mathcal{R}_T = \Theta(\frac{T}{\log T}).$ Theorem~\ref{thm:lower-bound} then implies that such algorithms must still incur worst-case allocation variability at least on the $\sqrt{T}$ scale. 
\begin{corollary}\label{cor:impossibility}
Any active learning policy must have  $\mathcal{S}_T = \omega(\sqrt{T})$.
\end{corollary}
Corollary~\ref{cor:impossibility} implies that any algorithm with even a modest intention to favor the better arm, hence achieving $o(T)$ regret, must incur worst-case allocation variability strictly larger than $\sqrt{T}$. This stands in stark contrast to the no-learning setting, where one can use non-adaptive allocation rules such as round-robin, for which $\mathcal{S}_T$ is essentially $0$.\footnote{When $T$ is not a multiple of $K$, one can randomize only the last few pulls; this introduces at most $O(1)$ allocation variability, which we ignore.} We term this unavoidable $\sqrt{T}$-scale allocation variability faced by any active-learning algorithm the \emph{price of learning}. 

\subsection{When does the worst case occur?}\label{sec: worst case}
Theorem~\ref{thm:lower-bound} does not specify the worst-case instances.  In general, when should we expect an algorithm to incur large allocation variability? To obtain a clearer picture, we focus on the two-armed Gaussian bandits $\nu(\Delta) = \{\mathcal{N}(0, 1), \mathcal{N}(-\Delta, 1)\}$, parameterized by the arm's mean reward gap $\Delta \ge 0.$ We refer to the second arm (mean $-\Delta$) the \textit{inferior arm}. The next result shows that as $T \to \infty$, any active learning algorithm must incur large allocation variability on an instance with $\Delta$ vanishingly small. 
\begin{corollary}\label{cor:hard-instance}
    {Consider an active learning algorithm with worst-case regret $\Theta (T^{\frac{1}{2}+\alpha})$, $\alpha\in[0,\frac{1}{2})$. Then there exists an absolute constant $c > 0$ such that  
    \begin{align*}
        \sup_{\Delta\in[0,c T^{\alpha-1/2}]}S_T(\Delta)=\Omega(T^{1-\alpha}) \ \ \text{as }T\rightarrow\infty.
    \end{align*}} 
\end{corollary}
The proof of Corollary~\ref{cor:hard-instance} follows directly from the argument of Theorem~\ref{thm:lower-bound}, and is deferred to Appendix~\ref{sec:lower-bound-proof}. Corollary~\ref{cor:hard-instance} is consistent with Example~\ref{example:intro}: large allocation variability must happen in some instance where the mean reward gap is small. Moreover, for minimax regret-optimal algorithms, we can strengthen Corollary~\ref{cor:hard-instance} and show that \emph{all} instances with statistically indistinguishable arms are worst case, in the sense that they incur $\Theta(T)$ allocation variability.

\begin{theorem}\label{thm:minimax-worst-case-interval}  
Consider minimax regret-optimal algorithms that satisfy $\mathcal{R}_T \le c_0 \sqrt{T}$ for some $c_0 > 0$. Fix any $c_1 > 0$. Then for any $\Delta \in \left[0, \frac{c_1}{\sqrt{T}}\right]$, there exists a constant $c_2$ that depends only on $c_0, c_1$, such that for  $T$ sufficiently large, $S_T(\Delta) \ge  c_2 T.$
\end{theorem}
The proof of Theorem~\ref{thm:minimax-worst-case-interval} follows the same overall strategy as that of Theorem~\ref{thm:lower-bound}. In particular, under a minimax regret-optimal algorithm, the expected pull count of the inferior arm must be sensitive to the gap $\Delta$, experiencing significant decrease from $T/2$ as $\Delta$ increases from zero to $O(1/\sqrt{T})$. We then apply a change-of-measure argument to show that, in this local regime, nearby instances are information-theoretically hard to distinguish. This translates into a large local change rate (sensitivity) of the expected pull count with respect to $\Delta$ whenever $\Delta= O(1/\sqrt{T})$, which in turn can be related to the large allocation variability in that range. The proof is provided in Appendix~\ref{sec: proof of minimax worst case}.

A special case covered by Theorem~\ref{thm:minimax-worst-case-interval} is $\Delta=0$, i.e., the instance with two identical arms. The theorem implies that this instance must incur maximal allocation variability under any minimax regret-optimal algorithm. Another important regime is when $\Delta$ is strictly bounded away from $0$, particularly, $\Delta \in \left[\frac{c_3}{\sqrt{T}},\,\frac{c_1}{\sqrt{T}}\right]$ for some $ 0<c_3<c_1.$ In this regime, a minimax regret-optimal algorithm not only attains its worst-case allocation variability according to Theorem~\ref{thm:minimax-worst-case-interval}, it also attains its worst-case regret $\Theta(\sqrt{T})$. In this sense, the regime $\Delta=\Theta(1/\sqrt{T})$ is worst-case for minimax regret-optimal algorithms with respect to both regret and allocation variability.
}

\section{Achievability and Pareto Frontier}\label{sec:algorithm}
{
Theorem~\ref{thm:lower-bound} establishes that the performance of any bandit algorithm, when evaluated in the two-dimensional regret--allocation-variability plane, must lie in the region $\mathcal{R}_T\cdot\mathcal{S}_T=\Omega\!\left(T^{3/2}\right).$ In this section, we show that this lower bound is essentially tight (up to polylog $T$ factors). To this end, we introduce a family of algorithms, which we refer to as \texttt{UCB-f}; see Algorithm~\ref{algo: ucb_f}. The algorithm \texttt{UCB-f} is a natural extension of the celebrated \texttt{UCB1} algorithm~\cite{auer2002finite}, which is recovered by taking $f(t)=\sqrt{2\log t}$.
\begin{algorithm}
    \caption{{\sf UCB-f}}
    \label{algo: ucb_f}
    \begin{algorithmic}[1]
        \State \textbf{Input:} Exploration function $f(\cdot)$.
        \State \textbf{Initialization:} For $t=1,\dots,K$, pull each arm $i$ once and set
        $N_{i,K}=1$, $\widehat\mu_{i, K}=X^{i}_1$ for all $i\in[K]$. 
        \State \textbf{for} $t=K+1,\dots,T$ \textbf{do}
        \State \quad Select $A_t\in \arg\max_{i\in[K]}\left\{\widehat\mu_{i,t-1}+\frac{f(t)}{\sqrt{N_{i,t-1}}}\right\}$.
        \State \quad Update $N_{i,t}\leftarrow N_{i,t-1}+\mathds{1}\{A_t=i\}$ for all $i\in[K]$.
        \State \quad Update $\widehat\mu_{i,t}\leftarrow \frac{1}{N_{i,t}}\left(\widehat\mu_{i, t-1} N_{i, t-1}+r_t\mathds{1}\{A_t=i\}\right)$ for all $i\in[K]$.
        \State \textbf{end for}
    \end{algorithmic}
\end{algorithm}
As discussed in Sections~\ref{sec: closer look}--\ref{sec: worst case}, \texttt{UCB1} attains near-minimax optimal regret, yet can exhibit large allocation variability when the arms are indistinguishable. The next theorem shows that tuning $f(\cdot)$ in \texttt{UCB-f} yields a smooth trade-off between the two objectives.

\begin{theorem}\label{thm:ucb-f}
For any exploration function $f(\cdot)$ such that both $f(t)/\log t$ and $\sqrt{t}/f(t)$ are non-decreasing in $t$, \texttt{UCB-f} achieves, over the instance class $\mathfrak{P}_{\mathrm{sg}}$,
\[
\mathcal{R}_T = O\!\left(T^{\frac12}f(T)\right),
\qquad
\mathcal{S}_T = O\!\left(\frac{T \log T}{f(T)}\right).
\]
\end{theorem}
The analysis of the \texttt{UCB-f} algorithm relies on a novel fluid approximation for the pull counts $(N_{i,t})_{i\in[K]}$ and a careful perturbation analysis that bounds the probability of deviation from this approximation; these techniques may be of independent interest. We defer the proof of Theorem~\ref{thm:ucb-f} to Appendix~\ref{sec: proof-ucb-f}.

\paragraph{Achievable region and Pareto frontier.} 
As an immediate implication of Theorem~\ref{thm:ucb-f}, every point $(\mathcal{R}_T,\mathcal{S}_T)$ on the curve $\mathcal{R}_T\cdot \mathcal{S}_T={\Theta}(T^{3/2})$ is achievable (up to polylog $T$ factors) by \texttt{UCB-f}. Combining with Theorem~\ref{thm:lower-bound}, we conclude that the boundary $\mathcal{R}_T\cdot \mathcal{S}_T=\tilde{\Theta}(T^{3/2})$ tightly characterizes the Pareto frontier of active-learning algorithms in the $(\mathcal{R}_T,\mathcal{S}_T)$ plane, as illustrated by \Cref{fig:pareto-frontier2} (a). 
\begin{corollary}\label{cor:pareto}
 The boundary $\mathcal{R}_T\cdot\mathcal{S}_T=\Theta(T^{3/2})$ forms the Pareto frontier (up to polylog factors) between worst-case regret and worst-case allocation variability for active learning algorithms.
\end{corollary}
This optimal regret--allocation-variability trade-off is attained by placing greater weight on exploration relative to exploitation, namely by tuning the exploration function $f(\cdot)$ in \texttt{UCB-f} to grow faster. Related design principles have appeared in prior work on mitigating regret tail risk (e.g., Section~9.2 of~\cite{lattimore2020bandit} and~\cite{fan2021fragility, simchi2023regret}; see also \Cref{sec:literature_review}). However, analyzing allocation variability requires techniques that differ substantially from existing regret tail analyses. The latter primarily controls the large-deviation exponent of the tail probabilities of $N_{i,T}$. In contrast, bounding allocation variability requires uniform control over $N_{i,T}$ across its entire range; see Appendix~\ref{sec: proof-ucb-f}. More broadly, similar stabilization ideas may be applied to other algorithms, such as \texttt{Thompson Sampling} (see~\cite{halder2025stable}) or \texttt{EXP3}. We conjecture that suitably modified variants can similarly achieve Pareto-efficient trade-offs as in \Cref{thm:ucb-f}, and leave a formal investigation to future work.

The next result formalizes that any $(\mathcal R_T, \mathcal S_T)$ pair above the Pareto frontier is essentially attainable. 
\begin{proposition}\label{prop:pareto-region-achievability}
    For any $\alpha\in(0,\frac{1}{2})$, $\beta\in(0,\alpha]$, there exists an algorithm attaining (over $\mathfrak{P}_{\text{sg}}$) 
    \begin{align*}
        \mathcal R_T=\Theta(T^{\frac{1}{2}+\alpha}), \ \ \ \mathcal S_T=\tilde \Theta(T^{1-\alpha+\beta}).
    \end{align*}
\end{proposition}
Proposition~\ref{prop:pareto-region-achievability}, together with Corollary~\ref{cor:pareto}, show that $\mathcal{R}_T\cdot\mathcal{S}_T=\Omega\!\left(T^{3/2}\right)$ characterizes the achievable region (up to polylog $T$ factors) of active-learning algorithms in the $(\mathcal{R}_T,\mathcal{S}_T)$ plane. The proof of Proposition~\ref{prop:pareto-region-achievability} is deferred to Appendix~\ref{sec:pareto_other_proofs}, where we construct a family of algorithms that run \texttt{UCB-f} for an initial phase and then switch to a round-robin rule.  Broadly, many heuristic policies, such as \texttt{Greedy} and \texttt{ETC} (explore-then-commit) are non-Pareto-efficient, and correspond to points that lie strictly in the interior of the achievable region.

}

\paragraph{Allocation versus regret variability.} 
Let $\hat R_T \triangleq \sum_{i\in[K]} \Delta_i\, N_{i,T}$ denote the \emph{random pseudo-regret}. We refer to the standard deviation of $\hat R_T$ as \emph{regret variability}. Although regret variability is closely related to allocation variability, their worst-case behaviors differ drastically, particularly in how each trades off with worst-case (expected) regret $\mathcal{R}_T$. The next corollary characterizes the worst-case behavior of regret variability and highlights the contrast in parallel to Corollary~\ref{cor:pareto}. let $\mathfrak{P}^2_{\mathrm{sg}}$ denote the subclass of two-armed bandit instances in $\mathfrak{P}_{\mathrm{sg}}$.
\begin{proposition}\label{prop:regret-variability}
Under any active-learning algorithm, we have
\[
\liminf_{T\rightarrow\infty}\frac{\log T}{\sqrt{T}}\cdot
\sup_{\nu\in\mathfrak{P}_{\mathrm{sg}}}
\sd\!\big(\hat R_T(\nu)\big)
>c_4
\]
for constant $c_4 > 0$ independent of the algorithm. Furthermore, over $\mathfrak{P}^2_{\mathrm{sg}}$, \texttt{UCB-f} achieves
\[
\liminf_{T\rightarrow\infty}
\frac{\sup_{\nu\in\mathfrak{P}^2_{\mathrm{sg}}}
{\sd\!\big(\hat R_T(\nu)\big)}}{\sqrt{T}\log T}
<c_5
\]
for constant $c_5 > 0$ and any exploration function $f(\cdot)$ satisfying the conditions of Theorem~\ref{thm:ucb-f}.
\end{proposition}
We provide the proof of Proposition~\ref{prop:regret-variability} in Appendix~\ref{sec:pareto_other_proofs}. The impossibility result follows from arguments similar to those used in the proof of Theorem~\ref{thm:lower-bound}, while the achievability result immediately follows from the proof of Theorem~\ref{thm:ucb-f}. We establish the achievability statement for two-armed bandits, as this suffices for illustrative and contrastive purposes.

Proposition~\ref{prop:regret-variability}, together with Theorem~\ref{thm:ucb-f}, indicates that there is essentially \textit{no} trade-off between worst-case expected regret and worst-case regret variability. Indeed, one may vary $\mathcal{R}_T$ to attain any rate $\tilde \Theta\left(T^{1/2}\,f(T)\right)$ by tuning $f(\cdot)$ in \texttt{UCB-f}, while the worst-case regret variability remains at the order $\tilde{\Theta}(\sqrt{T})$. In other words, the corresponding Pareto frontier (up to polylog $T$ factors) is a flat horizontal line, as plotted in Figure~\ref{fig:pareto-frontier2} (b). This is in sharp contrast to the nontrivial Pareto frontier between $\mathcal{R}_T$ and the worst-case allocation variability $\mathcal{S}_T$ characterized in Corollary~\ref{cor:pareto} and illustrated in Figure~\ref{fig:pareto-frontier2} (a).
\begin{figure}[h]
    \centering
    \begin{minipage}{0.49\linewidth}
        \centering
        \includegraphics[width=\linewidth]{pareto-frontier.png}
        \par\smallskip
        (a) $\mathcal R_T$ versus allocation variability.
    \end{minipage}
    \hfill
    \begin{minipage}{0.49\linewidth}
        \centering
        \includegraphics[width=\linewidth]{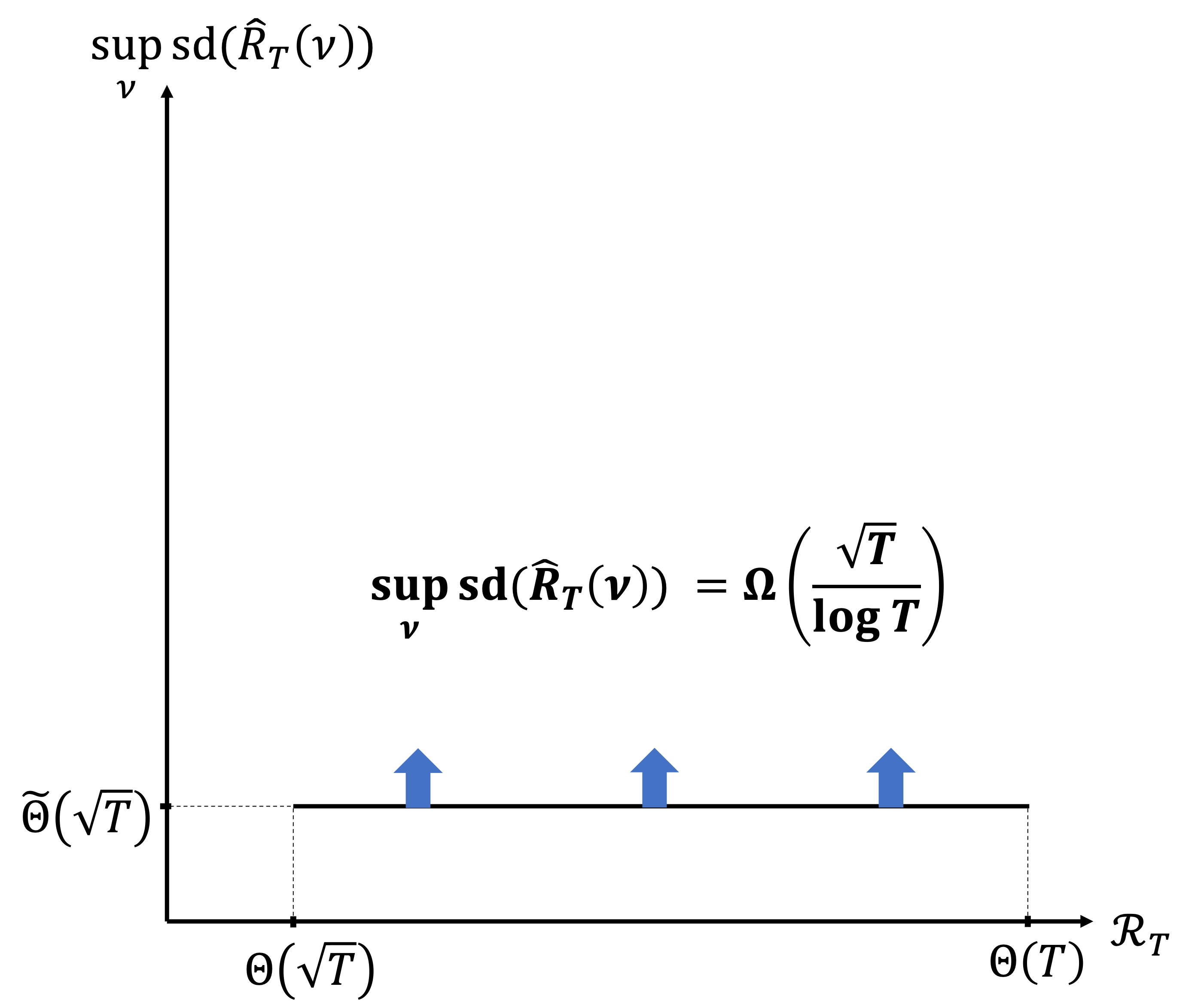}
        \par\smallskip
        (b) $\mathcal R_T$ versus regret variability.
    \end{minipage}
    \caption{Pareto frontiers of different objectives of interest.}
    \label{fig:pareto-frontier2}
\end{figure}

\section{Implications}\label{sec:implications}
Our theoretical findings have rich implications in various domains. In this section, we discuss how bandit algorithms' allocation variability can impact platform operations and statistical inference.

\paragraph{Implications on platform operations.} {Consider a platform who dynamically allocates online arriving demand to competing third-party agents with unknown service quality. Each demand allocation at time $t$ generates a random reward $r_t$ from the corresponding agent's reward distribution. The platform aims for a large total reward. At the same time, the platform also cares about the agents' utility to maintain its long term growth. The agents are risk-averse and get disutility if their demand variability is large. Following the classic mean-risk formulation (e.g., \cite{markowitz1952portfolio}), we express the platform's aggregate goal as maximizing expected total reward $\mathbb E\big[\sum_{s=1}^T r_s\big]$ minus the (largest among agents) agent disutility $S_T^\rho$ due to allocation variability, where $\rho\geq 0$ measures the agents' risk-aversion level (e.g., $\rho=0$ means the agents are risk-neutral). Equivalently, the platform aims to minimize the loss $R_T +S_T^\rho$.

Suppose the platform wants to incur a small aggregate loss in all possible instances $\nu$. Taking the robust optimization's minimax perspective, the platform's objective is to find an algorithm to minimize the worst-case loss 
\begin{align}\label{eq:platform-problem}
    \sup_{\nu\in\mathfrak{P}_{\mathrm{sg}}}\left(R_T(\nu)+(S_T(\nu))^\rho\right).
\end{align}
The following corollary of \Cref{thm:lower-bound} and \Cref{thm:ucb-f} establishes the lower bound of the platform's optimal loss, and prescribes an algorithm to achieve asymptotic near-optimality.
\begin{corollary}\label{cor:platform-value-lower-bound}
Consider the platform's problem in \eqref{eq:platform-problem} where $\rho\geq 0$. Under any bandit algorithm, we have
    \begin{align*}
        \liminf_{T\rightarrow\infty}\frac{\sup_{\nu\in\mathfrak{P}_{\mathrm{sg}}}\left(R_T(\nu)+(S_T(\nu))^\rho\right)}{\sqrt T\vee (T^{\frac{3\rho}{2+2\rho}}\wedge T}>c_6
    \end{align*}
    for constant $c_6>0$ independent of the algorithm. Furthermore, if $\rho\in[0,2]$, then the \texttt{UCB-f} algorithm with $f(t)=t^{\left(\frac{2\rho-1}{2+2\rho}\right)^+}\log t$ achieves
    \begin{align*}
        \sup_{\nu\in\mathfrak{P}_{\mathrm{sg}}}(R_T(\nu)+(S_T(\nu))^\rho)=\tilde\Theta(\sqrt T\vee T^\frac{3\rho}{2+2\rho}).
    \end{align*}
\end{corollary}
\Cref{cor:platform-value-lower-bound} states that the platform must at least incur a total loss of $\Theta\left(\sqrt T\vee \left(T^{\frac{3\rho}{2+2\rho}}\wedge T\right)\right))$ in the worst case. If $\rho>2$, the agents are very risk-averse, and the loss due to allocation variability dominates regret. In this case, the platform should give up learning (e.g., conduct round-robin) to avoid the superlinear loss from $\Omega(\sqrt T)$ allocation variability. If $\rho<\frac{1}{2}$, the agents are very insensitive to risk. This propels the platform to prioritize minimizing regret, with a minimax loss of $\Theta(\sqrt T)$ dominated by the regret term. If $\rho\in[\frac{1}{2},2]$, i.e., the agents are median level risk-averse, then the platform should carefully balance between regret and allocation variability. Our carefully tuned \texttt{UCB-f} algorithm can strike this balance. For example, when $\rho=1$, \texttt{UCB-f} with $f(t)=t^{\frac{1}{4}}\log t$ generates $\tilde \Theta(T^{\frac{3}{4}})$ worst-case regret and $\tilde \Theta(T^{\frac{3}{4}})$ worst-case allocation variability, which achieves the minimax loss of $\Theta(T^\frac{3}{4})$ up to a polylog $T$ factor in problem \eqref{eq:platform-problem}. The proof of \Cref{cor:platform-value-lower-bound} is in \Cref{app:sampling_instability}.
}

\paragraph{Implications for statistical inference.} 
Our results also have direct implications for statistical inference in MAB, a topic that has received growing attention in recent years (see literature review). Here the goal is to perform inference, e.g. to estimate an arm's mean reward, using data generated by a bandit algorithm. The central difficulty is that bandit algorithms are highly adaptive, and the data thus generated are non-\emph{i.i.d.} To ensure that standard inferential guarantees (developed under the \emph{i.i.d.} assumption) remain valid, a recent line of work \citep{han2024ucb,khamaru2024inference,praharaj2025instability} focuses on a key sufficient condition, termed \emph{sampling stability}, first introduced in \citet{lai1982least}. It is shown that under this condition, the sample mean reward $\hat\mu_{i,T} = \frac{1}{N_{i,T}}\sum_{s=1}^{N_{i,T}} X^i_s$ admits asymptotic normality \citep{han2024ucb}, thereby supporting CLT-based confidence intervals and hypothesis tests. We state this condition below (as in \citealp{praharaj2025instability}).
\begin{definition}[Sampling stability]\label{def:sampling-stability}
A bandit algorithm exhibits sampling stability over an instance class $\mathfrak{P}$ if for any instance $\nu\in\mathfrak{P}$ and any arm $i\in[K]$, there exist non-random scalars $n_{i,T}(\nu)$ such that $\frac{N_{i,T}(\nu)}{n_{i,T}(\nu)}\xrightarrow{p}1$, and $n_{i,T}(\nu)\to\infty$, as $T\to\infty.$
\end{definition}
A large body of work is devoted to determining whether various bandit algorithms exhibit sampling stability, and the current understanding remains largely algorithm-specific. In particular, \citet{praharaj2025instability} pose the following open question: \emph{Is it possible to design bandit algorithms that achieve both sampling stability and minimax-optimal regret?}

Our results resolve this open question with a negative answer. 
\begin{corollary}\label{cor:minimax-sampling-instability} Over $\mathfrak{P}_{\mathrm{sg}}$, no minimax regret-optimal algorithm that satisfies $\mathcal{R}_T \leq c_0 \sqrt{T}$ for some $c_0>0$ exhibits sampling stability.
\end{corollary}
The corollary follows immediately from \Cref{thm:lower-bound}--\ref{thm:minimax-worst-case-interval}, on instances with two identical Gaussian arms, together with the fact that allocation variability on the order of $\Theta(T)$ implies sampling instability in the sense of Definition~\ref{def:sampling-stability}. We provide the proof in \Cref{app:sampling_instability}. Beyond the negative result in Corollary~\ref{cor:minimax-sampling-instability}, even for bandit algorithms satisfying Definition~\ref{def:sampling-stability}, our notion of allocation variability provides a smooth, granular quantification of how stable they are. For example, the canonical \texttt{UCB1} algorithm has been shown to satisfy Definition~\ref{def:sampling-stability} and is therefore deemed stable; however, Corollary~\ref{cor:tradeoff-minimax-algorithms} suggests that its worst-case allocation variability is $\tilde{\Omega}(T)$, i.e., it is nearly unstable in the worst case. Consistent with this, numerical evidence in \citet{chen2025characterization} indicates a substantial difference between the distribution of $\sigma_i\sqrt{N_{i,T}}(\hat\mu_{i,T}-\mu_i)$ and the standard normal even at $T=10^5$, despite the fact that the former converges to the latter asymptotically. We conjecture that the large allocation variability of \texttt{UCB1} contributes to this slow CLT convergence, and we leave a precise characterization of the convergence rate as an intriguing open problem.

\section{Proof of \Cref{thm:lower-bound} }\label{sec:lower-bound-proof0}
In this section, we prove our main impossibility result, Theorem~\ref{thm:lower-bound}. Our proof of Theorem~\ref{thm:lower-bound} primarily focuses on the family of two-armed Gaussian bandit instances as in Section~\ref{sec: worst case}. Recall that such an instance is $\nu(\Delta) \triangleq \{\mathcal{N}(0,1), \mathcal{N}(-\Delta,1)\}$, parameterized by the mean reward gap $\Delta$ between the two arms. We adopt notations $S_T(\Delta), R_T(\Delta)$, whose meaning is clear from context. Denote by $g_T(\Delta) \triangleq \mathbb{E}_{\Delta}[N_{2, T}]$, i.e., the expected pull counts of the second (inferior) arm in instance $\Delta$. For simplicity, we assume $\Delta \le 2$. Hence $\nu(\Delta) \in \mathfrak{P}_{\text{sg}}(2,1)$. The generalization to arbitrary $\mathfrak{P}_{\text{sg}}(M,\sigma)$ is straightforward and we omit.

The proof relies on two key lemmas. First, for each gap $\Delta$, we relate the allocation variability $S_T(\Delta)$ to the local sensitivity of $g_T(\cdot)$ at $\Delta$ (Lemma~\ref{lem:S_lower_bound0}). Second, we show that $g_T(\cdot)$ cannot simultaneously vary slowly everywhere and yet decrease sufficiently as $\Delta$ grows, which is necessary to keep the expected regret $\mathcal{R}_T$ small (Lemma~\ref{lem:g_tradeoff0}). We now proceed to establish these two lemmas.
\begin{lemma}\label{lem:S_lower_bound0}
    Let $\Delta'\triangleq  \Delta + \frac{1}{\sqrt{g_T(\Delta)}}.$ Under any algorithm, it is true that
    \begin{align*}
         \max\left\{S_T(\Delta), S_T\left(\Delta'\right)\right\}\geq\frac{\sqrt{\frac{1}{2}e^{-\frac{1}{2}}}}{4}\cdot\left|g_T(\Delta)-g_T\left(\Delta'\right)\right|.
    \end{align*}
\end{lemma}
Intuitively, under instance $\Delta$, an algorithm pulls the inferior arm $g_T(\Delta)$ times on average. By CLT, the algorithm (up to time $T$) cannot effectively distinguish instance $\Delta$ from an alternative instance $\Delta'$ satisfying $|\Delta' - \Delta| \sim \frac{1}{\sqrt{g_T(\Delta)}}$ due to noisy reward samples. Thus under $\Delta$, the pull counts must ``stretch'' at times to behave as if the instance were $\Delta'$. The standard deviation of the pull counts under $\Delta$ should thus be at least $\left|g_T(\Delta)-g_T\left(\Delta'\right)\right|$. The formal proof of \Cref{lem:S_lower_bound0} utilizes the Bretagnolle-Huber inequality and the divergence decomposition. Details are provided in \Cref{sec:lower-bound-proof}.

We next establish a fundamental tension between
\[
\sup_{\Delta\in[0,1]} g_T(\Delta)\,\Delta
\qquad\text{and}\qquad
\sup_{\Delta\in[0,1]}\Bigl|\,g_T(\Delta)-g_T\!\Bigl(\Delta+\tfrac{1}{\sqrt{g_T(\Delta)}}\Bigr)\Bigr|.
\]
To make the former small, $g_T(\Delta)$ must decay rapidly as $\Delta$ increases. In contrast, making the latter small requires $g_T(\cdot)$ to vary slowly over $\Delta\in[0,1]$. These two requirements cannot hold simultaneously. We formalize this tension in the following lemma.

\begin{lemma}\label{lem:g_tradeoff0}
There exists an absolute constant $c$, such that for any active-learning algorithm
\begin{align*}   \liminf_{T\rightarrow\infty}\frac{\left(\sup_{\Delta\in[0, 1]}g_T(\Delta)\Delta\right)\left(\sup_{\Delta\in[0, 1]}\left|g_T(\Delta)-g_T\left(\Delta+\frac{1}{\sqrt{g_T(\Delta)}}\right)\right|\right)}{T^{\frac{3}{2}}}>c.
\end{align*}
\end{lemma}
\proof{Proof of \Cref{lem:g_tradeoff0}}
For each $T$, define a sequence of instances with $\Delta_{0, T} = 0$ and recursively 
\[\Delta_{m+1, T} = \Delta_{m, T} + \frac{1}{\sqrt{g_T(\Delta_{m, T})}}, m \ge 0.\]
By symmetry $g_T(0) = T/2.$ Let $\tau_{T} \triangleq \inf\{m \ge 1: g_T(\Delta_{m, T}) < \frac{T}{4}\}$ be the first time in the instance sequence that the expected pull count of the inferior arm drops below $T/4.$ Then for any $m \le \tau_T - 1$, it holds that $\frac{T}{4} \le g_T(\Delta_{m, T}) \le T$. This yields 
\[
\frac{m}{\sqrt{T}} \le \Delta_{m, T} \le \frac{2 m }{\sqrt{T}},  \qquad 0 \le m \le \tau_T.
\]
If $\Delta_{\tau_T, T} > 1$, then there exists a time $m \leq \tau_T$ such that $\Delta_{m-1, T} < 1$ yet $\Delta_{m, T} \ge 1$, due to monotonicity. Then $\Delta_{m-1, T} = \Delta_{m, T} - \frac{1}{\sqrt{g_T(\Delta_{m-1, T})}} \geq 1 - \frac{{2}}{\sqrt{T}} \ge \frac{1}{2}$ for $T \ge {16}$. Consequently, $\Delta_{m-1, T}\cdot g_T\left(\Delta_{m-1, T}\right) \ge \frac{1}{2}\times \frac{T}{4} = \frac{T}{8}$. Since the algorithm is active learning, this case should not occur infinitely often. Thus $\Delta_{\tau_T, T} \le 1$ for $T$ sufficiently large. This further implies $\Delta_{m, T} \le 1$ for all $m \le \tau_T$ for $T$ sufficiently large. 

Next, we consider two cases. In Case 1, suppose $g_T(\Delta_{\tau_T, T}) \le \frac{T}{20}$. Then combining with the definition of $\tau_T$, $g_T(\Delta_{\tau_{T - 1}, T}) - g_T(\Delta_{\tau_T, T}) \ge \frac{T}{4} - \frac{T}{20} = \frac{T}{5}.$ We thus have
\[
\sup_{\Delta \in[0, 1]}
\left|g_T(\Delta)-g_T\left(\Delta+\frac{1}{\sqrt{g_T(\Delta)}}\right)\right| \ge g_T(\Delta_{\tau_{T - 1}, T}) - g_T(\Delta_{\tau_T, T}) \ge \frac{T}{5}
\]
for $T$ sufficiently large.  Apply the standard minimax regret lower bound for Gaussian bandits (c.f. Thm. 15.2 in \cite{lattimore2020bandit}) $\sup_{\Delta\in[0, 1]} \Delta g_T(\Delta) \ge \frac{1}{27}\sqrt{T}$. Put together, we have 
\[
\left(\sup_{\Delta\in[0, 1]}g_T(\Delta)\Delta\right)\left(\sup_{\Delta\in[0, 1]}\left|g_T(\Delta)-g_T\left(\Delta+\frac{1}{\sqrt{g_T(\Delta)}}\right)\right|\right) \ge \frac{1}{135} T^{\frac{3}{2}}.
\]

Next we consider Case 2, i.e. $g_T(\Delta_{\tau_{T}, T}) > \frac{T}{20}$. It then must be $\frac{T}{20} < g_T(\Delta_{\tau_{T}, T}) < \frac{T}{4}.$ We have by telescoping 
\[
\sup_{\Delta\in[0, 1]}\left|g_T(\Delta)-g_T\left(\Delta+\frac{1}{\sqrt{g_T(\Delta)}}\right)\right| \ge \frac{\sum_{m = 0}^{\tau_T{-1}} \left(g_T(\Delta_{m, T}) - g_T(\Delta_{m+1, T})\right)}{\tau_T} \ge \frac{\frac{T}{2} - g_T(\Delta_{\tau_T, T})}{\tau_T} > \frac{T}{4 \tau_T}.
\]
Recall that $\Delta_{\tau_T, T} \ge \frac{\tau_T}{\sqrt{T}}$. Combining with $g_T(\Delta_{\tau_{T}, T}) > \frac{T}{{20}}$, we conclude that  $\sup_{\Delta\in[0, 1]}g_T(\Delta)\Delta \ge g_T(\Delta_{\tau_{T}, T})\Delta_{\tau_T, T}  > \frac{\tau_T \sqrt{T}}{20}.$ Put together, we have
\[
\left(\sup_{\Delta\in[0, 1]}g_T(\Delta)\Delta\right)\left(\sup_{\Delta\in[0, 1]}\left|g_T(\Delta)-g_T\left(\Delta+\frac{1}{\sqrt{g_T(\Delta)}}\right)\right|\right) \ge \frac{1}{80} T^{\frac{3}{2}}.
\]
Combining the above cases then completes the proof with the explicit constant $c = \frac{1}{135}.$ \qed
\endproof

\Cref{thm:lower-bound} is then immediate.
\begin{proof}{Proof of \Cref{thm:lower-bound}}
By definition and that $g_T(\Delta) \ge 1$ for $T \ge 2$ (since we assume the algorithm pulls each arm at least once),
\[
\mathcal{S}_T \geq \sup_{\Delta \in [0, 2]} S_T(\Delta) \ge \sup_{\Delta\in[0,1]}\max\left\{S_T(\Delta), S_T\left(\Delta + \frac{1}{\sqrt{g_T(\Delta)}}\right)\right\},
\]
Apply Lemma~\ref{lem:S_lower_bound0} to the above, we have
\[
\mathcal{S}_T \ge \frac{\sqrt{\frac{1}{2}e^{-\frac{1}{2}}}}{4}\cdot\sup_{\Delta \in[0, 1]}\left|g_T(\Delta)-g_T\left(\Delta+\frac{1}{\sqrt{g_T(\Delta)}}\right)\right|.
\]
Meanwhile, $\mathcal{R}_T \ge \sup_{\Delta \in[0, 1]}g_T(\Delta)\Delta$. Thus, apply \Cref{lem:g_tradeoff0},
    \begin{align*}
        \liminf_{T\rightarrow\infty}\frac{\mathcal R_T\mathcal S_T}{T^{\frac{3}{2}}}&\geq\liminf_{T\rightarrow\infty}\frac{\frac{1}{4}\sqrt{\frac{1}{2}e^{-\frac{1}{2}}}\left(\sup_{\Delta\in[0, 1]}g(\Delta)\Delta\right)\left(\sup_{\Delta\in[0, 1]}\left|g(\Delta)-g\left(\Delta+\frac{1}{\sqrt{g(\Delta)}}\right)\right|\right)}{T^{\frac{3}{2}}}\\
        &>\frac{1}{540}\sqrt{\frac{1}{2}e^{-\frac{1}{2}}},
    \end{align*}
    completing the proof. \qed
\end{proof}

\section{Concluding Remarks}\label{sec:conclusion}
In this paper, we identify and study a side effect of learning: the large variability in allocation. In the stochastic multi-armed bandit model, we show that there is a fundamental trade-off between any algorithm's allocation variability and expected regret, both evaluated in a worst-case sense, as long as the algorithm has an intention to learn. We further propose a family of \texttt{UCB-f} algorithms, a natural extension of \texttt{UCB1}, that achieves the Pareto frontier. Several natural questions arise from our analysis, which we outline below.
\begin{itemize}
    \item The lower bound in \Cref{thm:lower-bound} does not capture the dependence on the number of arms. Also the explicit $K$-dependence in Theorem~\ref{thm:ucb-f} is likely suboptimal. It remains open how the Pareto frontier scales with $K$.
    \item It is natural to ask whether other popular algorithms like \texttt{Thompson Sampling} can be modified to attain the full frontier. This necessary require analytical techniques different from the ones used for \texttt{UCB1-f}.
    \item Although our upper and lower bounds match up to polylog factors, it remains open whether the frontier can be sharpened further to remove these factors. 
    \item It would be interesting to extend our results to contextual bandits and, more broadly, to reinforcement learning, where exploration policies determine state–action visitation frequencies. The MAB model can be viewed as a special case of contextual bandits in which contexts are orthogonal and thus admit no information sharing across actions. We believe that with richer information sharing, the regret-allocation-variability trade-off may be less severe and allocational instability could be partially mitigated.

    \item Our findings point to intriguing connections with statistical inference under adaptive data collection. An open direction is to understand the \emph{rates} at which asymptotic normality emerges for estimators based on adaptively collected samples. Concretely, one explicit question is to characterize the convergence rate of $\sigma_i\sqrt{N_{i,T}}(\hat\mu_{i,T}-\mu_i)$ to the standard normal distribution as $T\to\infty$ for algorithms deemed stable, and in particular, to determine how this rate depends on the algorithm's allocation variability.

\end{itemize}

We view these as interesting directions and leave them for future investigation.


\newpage
\APPENDIX{}



This appendix collects proofs omitted from the main text. There are five sections. In \Cref{sec:lower-bound-proof} we prove \Cref{lem:S_lower_bound0} and \Cref{cor:hard-instance}. \Cref{sec: proof of minimax worst case} proves \Cref{thm:minimax-worst-case-interval}. In \Cref{sec: proof-ucb-f}, we provide the analysis of the proposed \texttt{UCB-f} family. \Cref{sec:pareto_other_proofs} presents other proofs of \Cref{sec:algorithm}. Finally, we prove \Cref{cor:platform-value-lower-bound} and \Cref{cor:minimax-sampling-instability} in \Cref{app:sampling_instability}.

\section{Proof of \Cref{lem:S_lower_bound0} and \Cref{cor:hard-instance}}\label{sec:lower-bound-proof}
{
We first prove Lemma~\ref{lem:S_lower_bound0}, which then completes the proof of Theorem~\ref{thm:lower-bound} (see Section~\ref{sec:lower-bound-proof0}). Recall that we consider two-armed Gaussian bandit instances. To simplify notations, we drop the subscript $T$, replacing $S_T(\cdot)$ by $S(\cdot)$ (the instance-dependent allocation variability), $N_{2, T}(\cdot)$ by $N(\cdot)$ (the number of pulls of the inferior arm) and $g_T(\cdot)$ by $g(\cdot)$ (the expected value of $N(\cdot)$). Let $P_{\nu\pi}$ denote the probability measure induced by the $T$-round interconnection of the algorithm $\pi$ and a bandit instance $\nu$. We sometimes omit the subscript $\pi$ and denote by $P_{\nu}$ when the algorithm is clear from the context. Let $D_{\text{KL}}(P\|Q)$ denote the KL divergence between two probability measures $P$ and $Q$. We then proceed with a sequence of lemmas.
}

{
\begin{lemma}\label{lem:S_lower_bound}
For any two-armed bandit instances $\nu$ and $\nu'$ and any fixed algorithm, suppose WLOG $g(\nu) \ge g(\nu')$, then 
    \begin{align*}        
    \max\{S(\nu),S(\nu')\}\geq\frac{g(\nu)-g(\nu')}{4}\sqrt{P_{\nu}\left(N(\nu)<\frac{g(\nu)+g(\nu')}{2}\right)+P_{\nu'}\left(N(\nu')\geq \frac{g(\nu)+g(\nu')}{2}\right)}.
    \end{align*}
\end{lemma}
\proof{Proof of \Cref{lem:S_lower_bound0}}
Observe that
 \begin{align*}
    S(\nu)&=\sqrt{\mathbb E_{\nu}[(N(\nu)-g(\nu))^2]}\\
    &\geq\sqrt{\mathbb E_{\nu}\left[(N(\nu)-g(\nu))^2\mathds 1\left\{N(\nu)<\frac{g(\nu)+g(\nu')}{2}\right\}\right]}\\
    &\geq \frac{g(\nu)-g(\nu')}{2}\sqrt{P_{\nu}\left(N(\nu)<\frac{g(\nu)+g(\nu')}{2}\right)}.
\end{align*}
Similarly, 
\[
S(\nu')\geq \frac{g(\nu)-g(\nu')}{2}\sqrt{P_{\nu'}\left(N(\nu')\geq \frac{g(\nu)+g(\nu')}{2}\right)}.
\]
Also $\max\{S(\nu),S(\nu')\}\ge\frac{S(\nu)+S(\nu')}{2}.$ Apply the bounds on $S(\nu), S(\nu')$ derived above, we have
\begin{align*}
\frac{S(\nu)+S(\nu')}{2}
    &\geq\frac{g(\nu)-g(\nu')}{4}\left(\sqrt{P_{\nu}\left(N(\nu)<\frac{g(\nu)+g(\nu')}{2}\right)}+\sqrt{P_{\nu'}\left(N(\nu')\geq \frac{g(\nu)+g(\nu')}{2}\right)}\right)\\
    &\geq\frac{g(\nu)-g(\nu')}{4}\sqrt{P_{\nu}\left(N(\nu)<\frac{g(\nu)+g(\nu')}{2}\right)+P_{\nu'}\left(N(\nu')\geq \frac{g(\nu)+g(\nu')}{2}\right)}.
\end{align*}
Combining the above completes the proof. \qed
\endproof
}

{
We recall the Bretagnolle-Huber inequality (cf. Thm. 14.2 in \cite{lattimore2020bandit}).
\begin{lemma}[Bretagnolle–Huber inequality]\label{lem: BH inequality}
  Let $P, Q$ be probability measures on the same measurable space, and let $A$ be an arbitrary event. Then,
  \[
  P(A) + Q(A^c) \ge \frac{1}{2}e^{-D_{\text{KL}}(P\|Q)}.
  \]
\end{lemma}
A direct application of 
Lemma~\ref{lem: BH inequality} with $P_{\nu\pi}$ and $P_{\nu'\pi}$, and the event $\left\{N < \frac{g(\nu) + g(\nu')}{2}\right\}$ yields
\begin{lemma}\label{lem:Bretagnolle-Huber}
Consider any $\nu,\nu'\in\mathcal P$ and a fixed algorithm $\pi$. Then we have
    \begin{align*}
        P_{\nu\pi}\left(N<\frac{g(\nu)+g(\nu')}{2}\right)+P_{\nu'\pi}\left(N\geq\frac{g(\nu)+g(\nu')}{2}\right)\geq\max\left\{\frac{1}{2}e^{-D_{\text{KL}}(P_{\nu\pi}\|P_{\nu'\pi})},\frac{1}{2}e^{-D_{\text{KL}}(P_{\nu'\pi}\|P_{\nu\pi})} \right\}.
    \end{align*}
\end{lemma} 
Here we remark that the two notations $P_{\nu\pi} \left(N<\frac{g(\nu)+g(\nu')}{2}\right)$ and $P_{\nu}\left(N(\nu)<\frac{g(\nu)+g(\nu')}{2}\right)$ refer to the same probability (i.e., the same event under the law induced by running algorithm $\pi$ on instance $\nu$). The next lemma follows from the divergence decomposition lemma (cf. Lemma 15.1 in \cite{lattimore2020bandit}), whose proof we omit.
\begin{lemma}\label{lem:divergence_decomposition}
    Consider any $\nu,\nu'\in\mathcal P$ and any algorithm $\pi$. WLOG assume that $\mu_1\geq\mu_2$ and $\mu_1'\geq\mu_2'$, namely, arm $1$ is the best arm. Then 
    \begin{align*}
        & D_{\text{KL}}(P_{\nu\pi}\|P_{\nu'\pi})=(T-g(\nu))D_{\text{KL}}(\nu_1\|\nu_1')+g(\nu)D_{\text{KL}}(\nu_2\|\nu_2'),\\
        & D_{\text{KL}}(P_{\nu'\pi}\|P_{\nu\pi})=(T-g(\nu'))D_{\text{KL}}(\nu_1'\|\nu_1)+g(\nu')D_{\text{KL}}(\nu_2'\|\nu_2).
    \end{align*}
\end{lemma}
Recall that for Gaussian distribution $\mathcal{N}(\mu, 1)$ and $\mathcal{N}(\mu', 1)$, their KL divergence is equal to $\frac{1}{2}(\mu- \mu')^2$. Consider two Gaussian measures $\nu = \{\mathcal{N}(\mu_1, 1), \mathcal{N}(\mu_2, 1) \}$ and $\nu' = \{\mathcal{N}(\mu_1', 1), \mathcal{N}(\mu_2', 1) \}$, where $\mu_1\geq\mu_2$, $\mu_1'\geq\mu_2'$ and $\mu_1=\mu_1'$. \Cref{lem:Bretagnolle-Huber}--\Cref{lem:divergence_decomposition} together yield the following.
\begin{lemma}\label{lem:tail-bound-gaussian}
    Consider any algorithm $\pi$. For the two Gaussian measures $\nu,\nu'$ described above with $|\mu_2-\mu_2'|=\frac{1}{\sqrt{g(\nu)}}$ or $|\mu_2-\mu_2'|=\frac{1}{\sqrt{g(\nu')}}$,
    \begin{align*}
        P_{\nu\pi}\left(N(\nu)<\frac{g(\nu)+g(\nu')}{2}\right)+P_{\nu'\pi}\left(N(\nu')\geq\frac{g(\nu)+g(\nu')}{2}\right)\geq \frac{1}{2}e^{-\frac{1}{2}}.
    \end{align*}
\end{lemma}
We now complete the proof of Lemma~\ref{lem:S_lower_bound0}.

\begin{proof}{Proof of Lemma~\ref{lem:S_lower_bound0}}
The lemma directly follows from Lemma~\ref{lem:S_lower_bound} and Lemma~\ref{lem:tail-bound-gaussian}.
\qed
\end{proof}
}

We now provide the proof of \Cref{cor:hard-instance}. 
\begin{proof}{Proof of \Cref{cor:hard-instance}}
    {The proof is similar to the proof of \Cref{thm:lower-bound}. In particular, \Cref{lem:S_lower_bound0} indicates that it suffices to find values of $\Delta_T=O(T^{\alpha-\frac{1}{2}})$ such that $g_T(\Delta_T)=\Theta(T)$ (so that $\Delta_T'=O(T^{\alpha-\frac{1}{2}})$ as well) and $\left|g_T(\Delta_T)-g_T(\Delta_T')\right|=\Omega(T^{1-\alpha})$. 
    
    We utilize the proof of \Cref{lem:g_tradeoff0} to find this $\Delta_T$. In particular, in the proof of \Cref{lem:g_tradeoff0}, there are two cases. In the first case, let $\Delta_{T}=\Delta_{\tau_T-1,T}$. Then we have $g_T(\Delta_T)\geq\frac{T}{4}$ and $\left|g_T(\Delta_T)-g_T(\Delta_T')\right|\geq\frac{T}{5}$ when $T$ is sufficiently large. Also, since the worst-case regret is $\Theta(T^{\frac{1}{2}+\alpha})$, it must be that $\Delta_T=O(T^{\alpha-\frac{1}{2}})$. In the second case, from the proof, there must exist $\Delta_T\in[0,\Delta_{\tau_T-1,T}]$ such that $g_T(\Delta_T)\geq\frac{T}{4}$ and $\left|g_T(\Delta_T)-g_T(\Delta_T')\right|\geq\frac{T}{4\tau_T}$, where $\tau_T$ satisfies $\frac{\tau_T\sqrt{T}}{20}<\sup_{\Delta\in[0,1]}g_T(\Delta)\Delta=\Theta(T^{\frac{1}{2}+\alpha})$. Then, it must be that $\tau_T=O(T^\alpha)$ hence $\left|g_T(\Delta_T)-g_T(\Delta_T')\right|=\Omega(T^{1-\alpha})$, and that $\Delta_T=O(T^{\alpha-\frac{1}{2}})$ since $g_T(\Delta_T)\geq\frac{T}{4}$ and the worst-case regret is $\Theta(T^{\frac{1}{2}+\alpha})$.} \qed
\end{proof}

\section{Proof of \Cref{thm:minimax-worst-case-interval}}\label{sec: proof of minimax worst case}
To prove Theorem~\ref{thm:minimax-worst-case-interval}, we first introduce some additional notation. Recall that $\nu_1(\Delta) = \mathcal{N}(0, 1),$ $\nu_2(\Delta) = \mathcal{N}(-\Delta, 1)$. Also recall that $\mathbf{X}^i \triangleq (X^i_1, \dots, X^i_T)$ is the rewards of the first $T$ pulls from arm $i = 1, 2$. Then $\mathbf{X}^i$ follows \textit{i.i.d.} $\mathcal{N}(\mu_i, 1)$ with $\mu_1 = 0$ and $\mu_2 = -\Delta.$ Let $\mathbf{X} = (\mathbf{X}^1, \mathbf{X}^2)$. The first $T$ rounds of interaction between any algorithm and the bandits must be within $\mathbf{X}$. We adopt notation $\mathbb{E}_{\Delta}, P_{\Delta}$ to emphasize the underlying model $\nu(\Delta)$. In particular, $f_{\Delta}(\cdot)$ denote the \emph{p.d.f.} of $\mathbf{X}$ under $\nu(\Delta)$. We also denote $N = N_{2, T}$ the pull counts of the second (inferior) arm. Before proving the theorem, we first state the standard change of measure identity, the proof of which we omit.
\begin{lemma}[Change of measure]\label{lem:change-of-measure}
    Suppose $h: \mathbb{R}^{ 2 T} \to \mathbb R$ is a measurable, bounded function that maps $\mathbf{X}$ to a real number. Then
    \[
    \mathbb{E}_{\Delta} [h(\mathbf{X})] = \mathbb{E}_{\Delta'}\left[h(\mathbf{X}) \cdot \left(\frac{f_{\Delta}(\mathbf{X})}{f_{\Delta'}(\mathbf{X})}\right)\right]
    \]
\end{lemma}
 We are now ready to prove \Cref{thm:minimax-worst-case-interval}.

\begin{proof}{Proof of \Cref{thm:minimax-worst-case-interval}}

Recall that $g_T(\Delta) = \mathbb{E}_\Delta[N]$, where $N$ is the pull count of the second arm, and we omit its dependence on the instance $\Delta$, since this dependence is made explicit in $\mathbb{E}_\Delta$. By definition, $R_T\left(\frac{3c_0}{\sqrt{T}}\right) \le \mathcal{R}_T  \le c_0 \sqrt{T}$ for $T$ sufficiently large (hence $\frac{3c_0}{\sqrt{T}} \le 2$). Recall that $R_T(\Delta) = \Delta g_T(\Delta)$. We thus have $g_T\left(\frac{3c_0}{\sqrt{T}}\right) \le \frac{T}{3}$. We also recall $g_T(0) = T/2.$ Denote $\frac{3 c_0}{\sqrt{T}} $ by $\Delta_0.$ For any $\Delta$, either $|g_T(\Delta) - g_T(0)| \ge \frac{T}{12}$ or $|g_T(\Delta) - g_T(\Delta_0)| \ge \frac{T}{12}$. WLOG let's assume $g_T(0) - g_T(\Delta) \ge \frac{T}{12}$ (in other cases the reasoning is similar and we omit). Under this condition, we derive the following bound on $S_T(\Delta)$.
\begin{align*}
    S_T(\Delta)&=\sqrt{\mathbb E_{\Delta}[(N -g_T(\Delta))^2]}\\
    &\geq\sqrt{\mathbb E_{\Delta}\left[(N-g_T(\Delta))^2\mathds 1\left\{N >\frac{g_T(\Delta)+g_T(0)}{2}\right\}\right]}\\
    &\geq \frac{|g_T(\Delta)-g_T(0)|}{2}\sqrt{P_{\Delta}\left(N>\frac{g_T(\Delta)+g_T(0)}{2}\right)}\\
    &\ge \frac{1}{24}\sqrt{P_{\Delta}\left(N>\frac{g_T(\Delta)+g_T(0)}{2}\right)} \cdot  T.
\end{align*}
Let $\mathcal{B} = \left\{N>\frac{g_T(\Delta)+g_T(0)}{2}\right\}$.  It suffices to show $P_{\Delta}\left(\mathcal{B} \right)$ is lower bounded by a constant. Importantly, under any algorithm, $N = N(\mathbf{X})$ is a measurable function of $\mathbf{X}$ satisfying the bound $N \le T$. Thus we may apply Lemma~\ref{lem:change-of-measure} with $\Delta' = 0$ and $h(\cdot) = \mathds{1}\{\mathcal{B} \}$:
\begin{align*}
P_\Delta\left(\mathcal{B}\right) = \mathbb{E}_{\Delta}\left[\mathds{1}\left\{\mathcal{B}\right\}\right] = \mathbb{E}_{0}\left[\mathds{1}\left\{\mathcal{B}\right\}\cdot \left(\frac{f_{\Delta}(\mathbf{X})}{f_{0}(\mathbf{X})}\right)\right].
\end{align*}
By definition of $\nu(\Delta)$, we have
\[
f_{\Delta}(\mathbf{X}) = \prod_{j = 1}^T \left(\frac{1}{2\pi}\exp\left(-\frac{(X^1_j)^2}{2}\right)\exp\left(-\frac{(X^2_j + \Delta)^2}{2}\right)\right).
\]
Plugging back to get
\begin{align*}
    P_\Delta\left(\mathcal{B}\right) &= \mathbb{E}_{0}\left[\mathds{1}\left\{\mathcal{B}\right\}  \prod_{j  = 1}^T\exp\left(-\frac{(X^2_j + \Delta)^2}{2}+ \frac{(X^2_j)^2}{2}\right) \right] \\
    &= \mathbb{E}_{0}\left[\mathds{1}\left\{\mathcal{B}\right\}  \cdot \exp\left(-  \frac{\Delta^2}{2} T- \Delta \sum_{j = 1}^T X^2_j\right) \right].
\end{align*}
Let $\mathcal{A} = \left\{\frac{1}{\sqrt{T}}\sum_{j = 1}^T X^2_j < 3\right\}$, we lower bound the above to get
\begin{align*}
  &\mathbb{E}_{0}\left[\mathds{1}\left\{\mathcal{B}\right\}  \cdot \exp\left(-  \frac{\Delta^2}{2} T- \Delta \sum_{j = 1}^T X^2_j\right) \right] \\
  & \ge \mathbb{E}_{0}\left[\mathds{1}\left\{\mathcal{B}\right\}  \cdot \exp\left(-  \frac{\Delta^2}{2} T- \Delta \sum_{j = 1}^T X^2_j\right) \mathds{1}\{\mathcal{A}\}\right] \\
  & \ge \mathbb{E}_{0}\left[\mathds{1}\left\{\mathcal{B}\right\}  \cdot \exp\left(-  \frac{\Delta^2}{2} T - 3\Delta \sqrt{T}\right) \mathds{1}\{\mathcal{A}\} \right]\\
  & \ge \exp\left(-  \frac{ c_1^2}{2} -  3 c_1\right) \mathbb{E}_{0}\left[\mathds{1}\left\{\mathcal{B}\right\} \cdot \mathds{1}\{\mathcal{A}\}\right]\\
  &= \exp\left(-  \frac{c_1^2}{2} - 3 c_1\right) \cdot P_{0}\left(\mathcal{B} \cap \mathcal{A}\right)
\end{align*}
where the second to the last inequality follows from $\Delta < \Delta_0 = \frac{c_1}{\sqrt{T}}.$ Observe that,
\begin{align*}
\frac{T}{2} &= \mathbb{E}_{0}\left[N \right]\\
&= \mathbb{E}_{0}\left[N\mathds{1}\{\mathcal{B}\} \right] + \mathbb{E}_{0}\left[N\mathds{1}\{\mathcal{B}^c\} \right]\\
& \le \frac{g_T(\Delta)+g_T(0)}{2}\cdot P_0\left(\mathcal{B}^c \right) + T \cdot P_0\left(\mathcal{B} \right)\\
& \le \left(\frac{T}{2} - \frac{T}{24}\right)\cdot P_0\left(\mathcal{B}^c \right) + T \cdot P_0\left(\mathcal{B} \right),
\end{align*}
which yields $P_0\left(\mathcal{B} \right) \ge \frac{1}{11}.$ On the other hand, under $\nu(0)$, $\sum_{j = 1}^T X^2_j \sim \mathcal{N}(0, T)$, thus $P_0(\mathcal{A}) =   \Phi(3)$, where $\Phi$ is the standard normal \emph{c.d.f.} Combining the above with an additional bound:
\[
P_{0}\left(\mathcal{B} \cap \mathcal{A}\right) = P_{0}\left(\mathcal{B} \right) + P_{0}\left(\mathcal{A}\right) - P_{0}\left(\mathcal{B} \cup \mathcal{A}\right) \ge \frac{1}{11} + \Phi(3) - 1  > \frac{1}{16}.
\]
Plugging back to the $S_T(\Delta)$ bound to get $S_T(\Delta) \ge \frac{e^{-  \frac{ c_1^2}{2} -  3 c_1}}{100} T$, completing the proof. \qed
\end{proof}

\section{Proof of \Cref{thm:ucb-f}}\label{sec: proof-ucb-f}
{
We introduce some additional notation. Throughout, fix a bandit instance $\nu = (\nu_1, \dots, \nu_K) \in \mathfrak{P}_{\text{sg}}(M, \sigma)$. Let $\widehat\mu_i(s) \triangleq \frac{1}{s}\sum_{t = 1}^s X^i_t$ denote the empirical mean of arm $i$ based on its first $s$ samples. Assume without loss of generality that the arm means are in a descending order, i.e. $\mu_1 \ge \mu_2 \ge \dots \mu_K$. In particular, the first arm has the largest mean, i.e. $\mu_1 = \mu^*$. Define $n_i^t,  i=1,2,...,K$, as the solution of the following fluid system of equations.
\begin{align}\label{eq:fluid-balance}
\begin{cases}
    \mu_1+\dfrac{f(t)}{\sqrt{n_1^t}}=\mu_i+\dfrac{f(t)}{\sqrt{n_i^t}}, & \forall i=2,\dots,K,\\[2mm]
    \sum_{i=1}^K n_i^t = t.
\end{cases}
\end{align}
For brevity, denote $n_i = n_i^T$ and $N_i = N_{i, T}$.
}

We state two auxiliary lemmas about $f(\cdot)$ and $n_i^t$.
\begin{lemma}\label{lem:key_ft_property}
    If the exploration function $f(\cdot)$ satisfies both $f(t)$ and $\frac{\sqrt{t}}{f(t)}$ increasing in $t$, then for any $\nu\in\mathcal P$, both $n_i^t$ and $\frac{\sqrt{n_i^t}}{f(t)}$ increase in $t$ for all $i=1,2,...,K$.
\end{lemma}
\begin{proof}{Proof}
    We first prove that $\frac{\sqrt{n_i^t}}{f(t)}$ increases in $t$ for all $i$. Suppose this is not true, then there exists $i\in\{1,2,...,K\}$ and $t_1<t_2$ such that $\frac{\sqrt{n_i^{t_1}}}{f(t_1)}>\frac{\sqrt{n_i^{t_2}}}{f(t_2)}$.
    
    By definition, for any $j\neq i$, we have
    \begin{align*}
        \mu_i+\frac{f(t_1)}{\sqrt{n_i^{t_1}}}=\mu_j+\frac{f(t_1)}{\sqrt{n_j^{t_1}}}, \ \ \mu_i+\frac{f(t_2)}{\sqrt{n_i^{t_2}}}=\mu_j+\frac{f(t_2)}{\sqrt{n_j^{t_2}}},
    \end{align*}
    hence by subtracting the above two equations, we get
    \begin{align*}
        \frac{f(t_1)}{\sqrt{n_i^{t_1}}}-\frac{f(t_2)}{\sqrt{n_i^{t_2}}}=\frac{f(t_1)}{\sqrt{n_j^{t_1}}}-\frac{f(t_2)}{\sqrt{n_j^{t_2}}}<0.
    \end{align*}
    That is, $\frac{\sqrt{n_i^{t_1}}}{f(t_1)}>\frac{\sqrt{n_i^{t_2}}}{f(t_2)}$ or equivalently $\frac{n_i^{t_1}}{f(t_1)^2}>\frac{n_i^{t_2}}{f(t_2)^2}$ for all $i=1,2,...,K$. Sum over the inequalities for all $i$, we get
    \begin{align*}
        \frac{t_1}{f(t_1)^2}=\frac{\sum_{i=1}^Kn_i^{t_1}}{f(t_1)^2}>\frac{\sum_{i=1}^Kn_i^{t_2}}{f(t_2)^2}=\frac{t_2}{f(t_2)^2},
    \end{align*}
    contradicting $\frac{\sqrt{t}}{f(t)}$ being increasing in $t$.

    The term $n_i^t$ increases in $t$ simply follows from the fact that both $f(t)$ and $\frac{\sqrt{n_i^t}}{f(t)}$ increase in $t$.
\end{proof}

\begin{lemma}\label{lem:n_i_asymptotics}
 If the exploration function $f(\cdot)$ satisfies both $f(t)$ and $\frac{\sqrt{t}}{f(t)}$ increasing in $t$, then $\inf_{\nu\in\mathfrak{P}_{\text{sg}}}n_i$ and $\inf_{\nu\in\mathfrak{P}_{\text{sg}}}\frac{n_2}{f(T)}$ both diverges as $T \to \infty$ for any $i=2,...,K$. Moreover, for such $f(\cdot)$, $\sup_{\nu\in\mathfrak{P}_{\text{sg}}}(n_i\Delta_i)=O\left(T^{\frac{1}{2}}f(T)\right)$ for any $i=2,...,K$. 
\end{lemma}
\begin{proof}{Proof}
    By straightforward algebra, one can get $\inf_{\nu\in\mathfrak{P}_{\text{sg}}}n_i=\Theta(f(T)^2).$  The rest of the proof follows and is omitted. \qed
\end{proof}

We now establish the following key lemma to prove \Cref{thm:ucb-f}.
{
\begin{lemma}\label{lem:N_i_tail_bound}
Consider \texttt{UCB-f} with an exploration function $f(\cdot)$ such that both
$f(t)/\log t$ and $\sqrt{t}/f(t)$ are non-decreasing in $t$. There exists a constant
$\underline{T}\equiv \underline{T}(M,\sigma,K,\alpha)$ such that, for any
$\nu\in\mathfrak{P}_{\mathrm{sg}}$, any $T\ge \underline{T}$, and any $i\in\{1,\dots,K\}$,
\begin{align*}
P_{\nu}\!\left(N_i-n_i>q\right)
&\le \frac{K-1}{4q^4},
\qquad \forall\, q\ge \frac{\log T}{f(T)}\left(n_i\wedge n_2\right),\\
P_{\nu}\!\left(N_i-n_i<-q\right)
&\le \frac{(K-1)^6}{4q^4},
\qquad \forall\, q\ge \frac{(K-1)\log T}{f(T)}\,n_2.
\end{align*}
\end{lemma}
}
\begin{proof}{Proof}
Throughout, fix $\nu \in \mathfrak{P}_{\text{sg}}$ and we omit the subscript $\nu$ for brevity. By Lemma~\ref{lem:key_ft_property}, $n_i^t \le n_i$ and $\frac{n_i}{f(T)} \ge \frac{n_i^t}{f(t)}$ for any $t \le T$ and $i \in [K].$ We first bound the right tail $P(N_i > n_i + q)$.
\paragraph{Bounding $P(N_i > n_i + q)$.} Observe that  $\{N_i > n_i +  q \}$ implies that there exists time $ t \in [n_i + q +1, T]$, such that $N_{i, t-1} = n_i + q$ and $A_{t} = i$. By the allocation rule of \texttt{UCB-f}, at time $t$, the UCB index of arm $i$ is maximal among all arms:
\[
\widehat\mu_i(n_i+q)+\frac{f(t)}{\sqrt{n_i+q}}
\ \ge\
\widehat\mu_j(N_{j,t-1})+\frac{f(t)}{\sqrt{N_{j,t-1}}}
\qquad \forall j\neq i.
\]
Meanwhile, $N_{i, t-1} = n_i +  q$ implies that $\sum_{j\neq i} N_{j,t-1} \ =\ t - 1 - (n_i+q).$ Since $n_i^s \le n_i$ for $s \le T$, it follow that $\sum_{j\neq i} N_{j,t-1}
\ =\ t - 1 - (n_i+q)
\ \le\ t-1 -(n_i^{t-1}+q)
\ =\ \sum_{j\neq i} n_j^{t-1}\ -\ q.$ With a pigeonhole argument, there must exist $j \ne i$ such that $N_{j, t-1} \le n_j^{t-1} - \frac{q}{K - 1}.$ Combining the above and a union bound, we have 
\begin{align}
P(N_i > n_i+q)
&\le \sum_{t=n_i+q+1}^{T}\ \sum_{j\neq i}\ \sum_{m= 1}^{n_j^{t-1}-\frac{q}{K-1}}
P\!\left(
\widehat\mu_i(n_i+q)+\frac{f(t)}{\sqrt{n_i+q}}
\ge
\widehat\mu_j(m)+\frac{f(t)}{\sqrt{m}}
\right).
\label{eq:right_tail_reduce}
\end{align}
Fix $(t,j,m)$ in the summation above.
Rearrange the index inequality by centering at the means:
\[
\big(\widehat\mu_i(n_i+q)-\mu_i\big) - \big(\widehat\mu_j(m)-\mu_j\big)
\ \ge\
(\mu_j-\mu_i) + f(t)\!\left(\frac{1}{\sqrt{m}}-\frac{1}{\sqrt{n_i+q}}\right).
\]
Using the defining property of $n^t$ and we can rewrite the above RHS as 
\[
f(t)\!\left(
\frac{1}{\sqrt{n_i^{t}}}-\frac{1}{\sqrt{n_i+q}}
+\frac{1}{\sqrt{m}}-\frac{1}{\sqrt{n_j^{t}}}
\right).
\]
By Lemma~\ref{lem:key_ft_property}, $n_i + q = n_i^T + q \ge n_i^{t} + q$. Also, $n_j^{t} \ge n_j^{t-1} \ge m + \frac{q}{K-1}$. We derive the bound
\begin{align}
&(\mu_j-\mu_i) + f(t)\!\left(\frac{1}{\sqrt{m}}-\frac{1}{\sqrt{n_i+q}}\right)\nonumber\\
& \ \ge\ 
f(t)\!\left(
\frac{1}{\sqrt{n_i^{t}+q}}-\frac{1}{\sqrt{n_i^{t}+q+\frac{q}{K-1}}}
+
\frac{1}{\sqrt{m}}-\frac{1}{\sqrt{m+\frac{q}{K-1}}}
\right),
\label{eq:rhs_lower}
\end{align}
where we used $\frac{1}{\sqrt{x}} - \frac{1}{\sqrt{x + q}} \ge \frac{1}{\sqrt{x + q}} - \frac{1}{\sqrt{x + q + \frac{q}{K-1}}}$ for $K \ge 2$ with $x$ set to $n_i^{t}$.
Let $Z =\big(\widehat\mu_i(n_i+q)-\mu_i\big) - \big(\widehat\mu_j(m)-\mu_j\big).$ By assumption and the property of sub-Gaussian random variables, $Z$ is $\sigma\sqrt{\frac{1}{n_i+q}+\frac{1}{m}}$-sub-Gaussian. Particularly,
\[
P(Z\ge x)\ \le\ \exp\!\left(-\frac{x^2}{2\sigma^2\left(\frac{1}{n_i+q}+\frac{1}{m}\right)}\right)
\ \le\ 
\exp\!\left(-\frac{x^2}{2\sigma^2\left(\frac{1}{n_i^t+q}+\frac{1}{m}\right)}\right),
\]
where we used $n_i^t\le n_i$ so $\frac{1}{n_i+q}\le \frac{1}{n_i^t+q}$.
Combining with \eqref{eq:rhs_lower} gives
\begin{equation}\label{eq:basic_exp_bound}
P\!\left(
\widehat\mu_i(n_i+q)+\frac{f(t)}{\sqrt{n_i+q}}
\ge
\widehat\mu_j(m)+\frac{f(t)}{\sqrt{m}}
\right)
\ \le\
\exp\!\left(
-\frac{f(t)^2}{2\sigma^2}\cdot
\frac{(A+B)^2}{\frac{1}{n_i^t+q}+\frac{1}{m}}
\right),
\end{equation}
where $A = \frac{1}{\sqrt{n_i^t + q}} - \frac{1}{\sqrt{n_i^t + q  +\frac{q}{K-1}}}$ and $B = \frac{1}{\sqrt{m}} - \frac{1}{\sqrt{m + \frac{q}{K-1}}}$. Next we lower bound the exponent in the above RHS. Define, for $x>0$ and $a>0$,
\[
\phi(x, a):=x\left(\frac{1}{\sqrt{x}}-\frac{1}{\sqrt{x+a}}\right)^2
=\left(1-\sqrt{\frac{x}{x+a}}\right)^2.
\]
A direct derivative check shows $\phi(\cdot, a)$ is decreasing in $x$ for each fixed $a>0$. Note that $(n_i^t + q)A^2 = \phi\left(n_i^t + q, \frac{q}{K-1}\right)$ and $m B^2 = \phi\left(m, \frac{q}{K-1}\right)$. Consider two cases.

\emph{Case 1: $m\ge n_i^t + q$.} Then $\frac{1}{n_i^t + q}+\frac{1}{m}\le \frac{2}{n_i^t + q}$ and $(A+B)^2\ge A^2$, hence
\[
\frac{(A+B)^2}{\frac{1}{n_i^t + q}+\frac{1}{m}}
\ \ge\ \frac{A^2}{2/(n_i^t + q)}
\ =\ \frac{n_i^t + q}{2}A^2
\ =\ \frac12\,\phi\left(n_i^t + q, \frac{q}{K-1}\right).
\]

\emph{Case 2: $m<n_i^t + q$.} Then $\frac{1}{n_i^t + q}+\frac{1}{m}\le \frac{2}{m}$ and $(A+B)^2\ge B^2$, hence
\[
\frac{(A+B)^2}{\frac{1}{n_i^t + q}+\frac{1}{m}}
\ \ge\ \frac{B^2}{2/m}
\ =\ \frac{m}{2}B^2
\ =\ \frac12\,\phi\left(m;\frac{q}{K-1}\right)
\ \ge\ \frac12\,\phi\left(n_i^t + q, \frac{q}{K-1}\right),
\]
where the last step uses that $\phi(\cdot, a)$ is decreasing and $m<n_i^t +q$. Combining, we have shown that for every $m \ge 1$, 
\begin{equation*}
\frac{(A+B)^2}{\frac{1}{n_i^t + q}+\frac{1}{m}}
\ \ge\ \frac{1}{2}\,\phi\left(n_i^t + q, \frac{q}{K-1} \right).
\end{equation*}
Using $1-\sqrt{1-x}\ge x/2$ for $x\in[0,1]$ and the definition of $\phi$, we further bound the above RHS by
\begin{align*}
    \frac{(A+B)^2}{\frac{1}{n_i^t + q}+\frac{1}{m}}
\ \ge\ \frac{1}{8}\Big(\frac{q}{(K-1)n_i^t+Kq}\Big)^2 \ \ge \ \frac{1}{32}\Big(\frac{q^2}{(K-1)^2(n_i^t)^2}\ \wedge\ \frac{1}{K^2}\Big).
\end{align*}
Plugging back to \eqref{eq:basic_exp_bound}, we derive the following bound on the exponent
\begin{align}\label{eq:exp-exponent-bound0}
  \frac{f(t)^2}{2\sigma^2}\cdot \frac{(A + B)^2}{\frac{1}{n_i^t + q} + \frac{1}{m}} & \ge \frac{f(t)^2}{64\sigma^2} \Big(\frac{q^2}{(K-1)^2(n_i^t)^2}\ \wedge\ \frac{1}{K^2}\Big).
\end{align}
Recall we assume $q \ge \frac{\log T}{f(T)}\left(n_i \wedge n_2\right).$ When $i \ne 1$, this condition is $q \ge \frac{n_i \log T}{f(T)}.$ By Lemma~\ref{lem:key_ft_property}, $\frac{n_i}{f(T)} \ge \frac{n_i^t}{f(t)}$. Combining, we have 
\[
q^2 \ge \left(\frac{n_i \log T}{f(T)}\right)^2 \ge \left( \frac{n_i^t \log T}{f(t)}\right)^2 \ge \frac{(n_i^t \log t)^2}{f(t)^2}.
\]
Plugging back to \eqref{eq:exp-exponent-bound0} and we have the RHS is at least $\frac{(\log t)^2}{64\sigma^2 K^2}$. By choosing $\underline{T} = \exp\left(384 \sigma^2 K^2\right)$, for all $t \ge \underline{T}$ we have $\frac{(\log t)^2}{64\sigma^2 K^2}  \ge 6 \log t$. Plugging back to \eqref{eq:basic_exp_bound} we have for all $t \ge \underline{T}$, 
\[
P\!\left(
\widehat\mu_i(n_i+q)+\frac{f(t)}{\sqrt{n_i+q}}
\ge
\widehat\mu_j(m)+\frac{f(t)}{\sqrt{m}}
\right)
\ \le\ t^{-6}.
\]
Now plugging the above back to \eqref{eq:right_tail_reduce}, using the crude bounds
$n_j^{t-1}-\frac{q}{K-1} \le t$ and $|\{j\neq i\}|=K-1$, we obtain for $n_i \ge \underline{T}$,
\begin{align*}
P(N_i > n_i+q)
&\le \sum_{t=n_i+q +1}^{T}\ \sum_{j\neq i}\ \sum_{m=1}^{n_j^{t-1}-\frac{q}{K-1}} t^{-6}
\ \le\ \sum_{t=n_i+q +1}^{T}(K-1)\cdot t\cdot t^{-6}\\
&=\ (K-1)\sum_{t=n_i+q +1}^{T} t^{-5}
\ \le\ (K-1)\int_{n_i+q}^{\infty} x^{-5}\,dx
\ =\ \frac{K-1}{4(n_i+q)^4}
\ \le\ \frac{K-1}{4q^4},
\end{align*}
where the last inequality uses $n_i\ge 0$ (which holds for all large $T$). Also, by Lemma~\ref{lem:n_i_asymptotics}, $n_i = \omega(1)$ as $T \to \infty$. hence for sufficiently large $T$ depending on $M, \sigma, K$, condition $n_i  \ge \underline{T}$ can hold uniformly over the instance class $\mathfrak{P}_{\text{sg}}$. This proves the right-tail bound for $i\ge 2$ when $q\ge  \frac{\log T}{f(T)}n_i$.

For $i=1$, the exact same argument goes through with $n_i$ replaced by $n_2$ in the threshold. Here such a replacement is possible since in treating the exponent of \eqref{eq:basic_exp_bound}, we have $m \le n_j^{t-1} -\frac{q}{K-1} < n_2^{t} + q$ for any $j = 1, \dots, K$ (since the arms are in descending order). The same two-case splitting argument would yield 
\[
\frac{(A+B)^2}{\frac{1}{n_1^t + q}+\frac{1}{m}}
\ >\ \frac{1}{2}\,\phi\left(n_2^t + q, \frac{q}{K-1} \right),
\]
and the rest of the proof follows. Combining the above completes the proof of the right tail for all $i \in [K]$.

\paragraph{Bounding $P(N_i < n_i - q)$.}  We now turn to the left tail. Fix $i\in[K]$ and $q>0$. We make use of the right tail above to establish a crude bound. In particular, since $\sum_{k=1}^K N_k = \sum_{k=1}^K n_k = T$, we have
\[
\{N_i<n_i-q\}\ \subseteq\ \left\{\sum_{j\neq i}(N_j-n_j)>q\right\}
\ \subseteq\ \bigcup_{j\neq i}\left\{N_j-n_j>\frac{q}{K-1}\right\}.
\]
Therefore, $P(N_i-n_i<-q)
\ \le\ \sum_{j\neq i}P\!\left(N_j-n_j>\frac{q}{K-1}\right).$ If $q\ge \frac{(K-1)\log T}{f(T)}n_2$, then for every $j\neq i$ we have
$\frac{q}{K-1}\ge \frac{\log T}{f(T)}n_2\ge \frac{\log T}{f(T)}\left(n_j\wedge n_2\right)$, so the right-tail bound gives
\[
P\!\left(N_j-n_j>\frac{q}{K-1}\right)
\ \le\ \frac{K-1}{4\left(\frac{q}{K-1}\right)^4}
\ =\ \frac{(K-1)^5}{4q^4}.
\]
Summing over at most $K-1$ arms yields
\[
P(N_i-n_i<-q)\ \le\ (K-1)\cdot \frac{(K-1)^5}{4q^4}\ =\ \frac{(K-1)^6}{4q^4}
\]
as claimed.

This completes the proof.   \qed
\end{proof}

We are now ready to prove \Cref{thm:ucb-f}.

\begin{proof}{Proof of \Cref{thm:ucb-f}}
Fix $\alpha\in[0,\tfrac12)$. Recall $n_i = n_i^T$ is the solution of the fluid system at horizon $T$, and we use $N_i = N_{i,T}$ interchangeably.
\paragraph{Allocation variability.}
By \Cref{lem:N_i_tail_bound}, for any $T\ge \underline T$, any $\nu \in \mathfrak{P}_{\text{sg}}$ and any $i\in[K]$,
\[
P\big(|N_i-n_i|>q\big)\ \le\ \frac{(K-1)^6}{2q^4},
\qquad \forall q\ge \frac{(K-1)\log T}{f(T)}n_2.
\]
Fix $q_0\triangleq \frac{(K-1)\log T}{f(T)}n_2$. Using
$\E[Z^2]=\int_0^\infty 2u\,P(|Z|>u)\,du$, we obtain
\begin{align*}
\E\big[(N_i-n_i)^2\big]
&= \int_0^{q_0} 2u\,P\big(|N_i-n_i|>u\big)\,du
   +\int_{q_0}^{\infty} 2u\,P\big(|N_i-n_i|>u\big)\,du \\
&\le q_0^2 + \int_{q_0}^{\infty} 2u\cdot \frac{(K-1)^6}{2u^4}\,du
 \;=\; q_0^2 + \frac{(K-1)^6}{2}\int_{q_0}^{\infty} u^{-3}\,du \\
&= \frac{(K-1)^2 (\log T)^2 n_2^2}{f(T)^2}+\frac{(K-1)^4 f(T)^2}{4n_2^2(\log T)^2}.
\end{align*}
By \Cref{lem:n_i_asymptotics}, $\inf_{\nu\in\mathfrak{P}_{\text{sg}}} n_2/f(T)\to\infty$ as $T\to\infty$.
Hence there exists $T_0=T_0(K,\alpha)$ such that for all $T\ge T_0$ and all $\nu\in\mathfrak{P}_{\text{sg}}$,
\[
\frac{(K-1)^4 f(T)^2}{4n_2^2(\log T)^2}\ \le\ \frac{(K-1)^2 (\log T)^2 n_2^2}{f(T)^2}.
\]
Therefore, for all $T\ge \underline{\underline T}\triangleq \underline T\vee T_0$, $\E\big[(N_i-n_i)^2\big]\ \le\ 2\frac{(K-1)^2 (\log T)^2 n_2^2}{f(T)^2},$ and hence 
\[
\sqrt{\var(N_i)}\ \le\ \sqrt{\E[(N_i-n_i)^2]}\ \le\ \sqrt{2}(K-1)\frac{\log T}{f(T)}n_2.
\]
Since $n_2=O(T)$, we obtain 
\[
\mathcal S_T
=\sup_{\nu\in\mathfrak{P}_{\text{sg}}}\max_{i\in[K]}\sqrt{\var_\nu(N_{i,T})}
= O\left(\frac{T\log T}{f(T)}\right).
\]

\paragraph{Regret.}
We bound $\E N_i$ via the right tail. By \Cref{lem:N_i_tail_bound}, for $T\ge \underline T$ and
any $i\in[2, K]$,
\[
P(N_i-n_i>q)\ \le\ \frac{K-1}{4q^4},
\qquad \forall q\ge \frac{\log T}{f(T)}n_i.
\]
Taking $q=n_i$ and using $\E[Y]=\int_0^\infty P(Y>u)\,du$ for $Y\ge 0$, we get
\begin{align*}
\E N_i
&\le n_i + \E[(N_i-n_i)_+]
 \le n_i + \int_0^{n_i} 1\,du + \int_{n_i}^{\infty} P(N_i-n_i>u)\,du \\
&\le 2n_i + \int_{n_i}^{\infty} \frac{K-1}{4u^4}\,du  \;=\;2n_i + \frac{K-1}{12\,n_i^3}.
\end{align*}
Since by Lemma~\ref{lem:n_i_asymptotics}, $\inf_{\nu\in\mathfrak{P}_{\text{sg}}}n_i\to\infty$, there exists $T_0'=T_0'( M, K)$ such that $\frac{K-1}{12n_i^3}\le 2n_i$ for all $i$ and all $\nu$ whenever $T\ge T_0'$. Thus for $T\ge \underline T\vee T_0'$, $\E_{\nu} N_i \le 4n_i,$ for all $i\in[2, K]$ and $\nu \in \mathfrak{P}_{\text{sg}}$ Therefore,
\[
\sup_{\nu\in\mathfrak{P}_{\text{sg}}} R_T(\nu)
=\sup_{\nu\in\mathfrak{P}_{\text{sg}}}\sum_{i=2}^K \Delta_i\,\E_{\nu} N_i
\le 4\sup_{\nu\in\mathfrak{P}_{\text{sg}}}\sum_{i=2}^K \Delta_i\,n_i
\le 4\sum_{i=2}^K \sup_{\nu\in\mathfrak{P}_{\text{sg}}}(n_i\Delta_i)
=O\!\left(T^{\frac12}f(T)\right),
\]
where the last step again uses \Cref{lem:n_i_asymptotics}. This completes the proof. \qed
\end{proof}

\section{Other proof of \Cref{sec:algorithm}}\label{sec:pareto_other_proofs}
We first prove \Cref{prop:pareto-region-achievability}.
\begin{proof}{Proof of \Cref{prop:pareto-region-achievability}}

Fix $\alpha\in(0,\frac{1}{2})$ and $\beta\in(0,\alpha]$. Consider the following policy for horizon length $T$. At $t=1,...,T-\lceil T^{\frac{1}{2}+\alpha}\rceil$, apply \texttt{UCB-f} with exploration function $f(t)=t^{\alpha-\beta}\log t$. In the remaining time $t=T-\lceil T^{\frac{1}{2}+\alpha}\rceil+1,...,T$, play each arm round-robin. Then by \Cref{thm:ucb-f}, it is straightforward to see that under this policy, 
    \begin{align*}
        R_T=\Theta(T^{\frac{1}{2}+\alpha}), \ \ S_T=\tilde \Theta(T^{1-\alpha+\beta}).
    \end{align*}
\end{proof}

We now prove \Cref{prop:regret-variability}.
\begin{proof}{Proof of \Cref{prop:regret-variability}}
    We first prove the lower bound.

    \paragraph{Lower bound: $\sup_{\nu}\sqrt{\var \hat R_T(\nu)} = \Omega\left(\frac{\sqrt{T}}{\log T}\right)$.}
    Similar to the proof of Theorem~\ref{thm:lower-bound} in Section~\ref{sec:lower-bound-proof0} and Appendix~\ref{sec:lower-bound-proof}, we focus on the family of instances $\nu(\Delta) = \{\mathcal{N}(0, 1), \mathcal{N}(-\Delta, 1)\}$. By \Cref{lem:S_lower_bound}, we know that 
    \begin{align*}        \sup_{\nu\in\mathfrak{P}_{\text{sg}}}\sqrt{\var_{\nu}(\hat R_T(\nu))}&\ge\sup_{\Delta \in [0 ,2]}\Delta S_T(\Delta)\\
        &\geq\sup_{\Delta,\Delta'\in[0,2], g_T(\Delta)\geq g_T(\Delta')}\Bigg\{\frac{(\Delta\vee\Delta')(g_T(\Delta)-g_T(\Delta'))}{4}\\
        &\qquad\qquad\qquad\cdot\sqrt{P_{\Delta}\left(N_T<\frac{g_T(\Delta)+g_T(\Delta')}{2}\right)+P_{\Delta'}\left(N_T\geq\frac{g_T(\Delta)+g_T(\Delta')}{2}\right)}\Bigg\}.
    \end{align*}
    By \Cref{lem:Bretagnolle-Huber}--\ref{lem:tail-bound-gaussian}, $\sqrt{P_{\Delta}\left(N_T<\frac{g_T(\Delta)+g_T(\Delta')}{2}\right)+P_{\Delta'}\left(N_T\geq\frac{g_T(\Delta)+g_T(\Delta')}{2}\right)}$ is lower bounded by a constant when $\Delta'-\Delta=\frac{1}{\sqrt{g_T(\Delta)}}$. Therefore it suffices to show
    \begin{align*}      \liminf_{T\rightarrow\infty}\frac{\sup_{\Delta\in[0,1]}\left(\Delta+\frac{1}{\sqrt{g_T(\Delta)}}\right)\Big|g_T(\Delta)-g_T\left(\Delta+\frac{1}{\sqrt{g_T(\Delta)}}\right)\Big|}{\sqrt{T}/\log T}>0.
    \end{align*}
    The proof is similar to that of Lemma~\ref{lem:g_tradeoff0}. In particular, define    
    \begin{align*}
        \xi(T):=\sup_{\Delta\geq 0}\left(\Delta+\frac{1}{\sqrt{g_T(\Delta)}}\right)\left|g_T(\Delta)-g_T\left(\Delta+\frac{1}{\sqrt{g_T(\Delta)}}\right)\right|.
    \end{align*}
    Then    \begin{align}\label{eq:g_different_uppbound2}
        g_T(\Delta)-g_T\left(\Delta+\frac{1}{\sqrt{g_T(\Delta)}}\right)\leq\frac{\xi(T)}{\Delta+\frac{1}{\sqrt{g_T(\Delta)}}}, \ \ \forall\Delta\geq 0.
    \end{align}
    Define 
    \begin{align}\label{eq:Delta_sequence2}
        \Delta_{0, T}=0, \ \ \Delta_{m+1, T}=\Delta_{m, T}+\frac{1}{\sqrt{g_T(\Delta_{m, T})}}, \forall m \ge 0,
    \end{align}
    Since $g_T(\Delta) \le T$ for any $\Delta \ge 0$, we have $\Delta_{m, T}\geq\frac{m}{\sqrt T}, \ \ \forall m \ge 0$. Combining this bound with  \eqref{eq:g_different_uppbound2} and telescoping, for all $n\geq 1$ such that $\Delta_{n, T} \le 1$
    \begin{align}\label{eq:telescoping}
        g_T(0)-g_T(\Delta_{n, T})&=\sum_{m=1}^{n}(g_T(\Delta_{m-1, T})-g_T(\Delta_{m, T})) \leq \sum_{m=1}^n\frac{\xi(T)}{\Delta_{m, T}}\leq\xi(T)\sum_{m=1}^n\frac{\sqrt{T}}{m}.
    \end{align}
    Consider $\tau_T = \inf\{m \ge 1: g_T(\Delta_m) < \frac{T}{4} \}$. Then similar as in the proof of Lemma~\ref{lem:g_tradeoff0}, we must have $\Delta_{\tau_T, T} \le 1$ for $T$ sufficiently large, due to active learning. This also implies $\tau_T \le \sqrt{T},$ since $\Delta_{m, T} \ge \frac{m}{\sqrt{T}}$. We thus apply \eqref{eq:telescoping} to get
    \begin{align*}
        g_T(0) - g_T(\Delta_{\tau_T, T}) \le \xi(T) \sum_{m = 1}^{\tau_T} \frac{\sqrt{T}}{m} \le \xi(T)\sqrt{T}\left(\log\tau_T + 1 \right) \leq \xi(T)\sqrt{T}\left(\frac{1}{2}\log T + 1 \right),
    \end{align*}
    which is further bounded by $\sqrt{T}\log T$ for $T$ sufficiently large. On the other hand
    $g_T(0) - g_T(\Delta_{\tau_T, T}) \ge \frac{T}{2} - \frac{T}{4}$, due to the symmetry assumption at $\Delta = 0$ and the definition of $\tau_T$. Combining the above we conclude that 
    $\xi(T) \ge \frac{\sqrt{T}}{4 \log T}$ for $T$ sufficiently large, completing the proof.   

    \paragraph{Achievability by \texttt{UCB-f}.} From the proof of Theorem~\ref{thm:ucb-f}, for any $\nu \in \mathfrak{P}_{\text{sg}}$ and any $i \in [K]$, we have 
    \[
    \E_{\nu}[(N_i - n_i)^2] \le \frac{2(K-1)^2(\log T)^2 n_2^2 }{f(T)}
    \]
    In the case of two arms (where arm 2 is the inferior arm), we have 
    $\hat R_T(\nu) = \Delta N_2$. And,
    \begin{align*}
        \sqrt{\var_{\nu}(\hat R_T(\nu))} = \sqrt{\var_{\nu}(\Delta N_2)} \le \frac{\sqrt{2}(K-1) \log T}{f(T)}\Delta n_2.
    \end{align*}
    Apply Lemma~\ref{lem:n_i_asymptotics} by taking superior over all two armed instances, we thus have the worst case regret variability of \texttt{UCB-f} under any choice of $f(t)$ satisfying the conditions in Theorem~\ref{thm:ucb-f} is at most $O(\sqrt{T}\log T)$, completing the proof.   \qed
    \end{proof}

\section{Proof of \Cref{sec:implications}}\label{app:sampling_instability}
We prove \Cref{cor:platform-value-lower-bound} and \Cref{cor:minimax-sampling-instability} in this appendix.
\begin{proof}{Proof of \Cref{cor:platform-value-lower-bound}}
    Note that
    \begin{align*}
        \sup_{\nu\in\mathfrak{P}_{\mathrm{sg}}}\left(R_T(\nu)+(S_T(\nu))^\rho\right)\geq \mathcal R_T, \ \ \sup_{\nu\in\mathfrak{P}_{\mathrm{sg}}}\left(R_T(\nu)+(S_T(\nu))^\rho\right)\geq\sup_{\nu\in\mathfrak{P}_{\mathrm{sg}}}(S_T(\nu))^\rho=\mathcal S_T^\rho.
    \end{align*}
    By the minimax-regret lower bound, we have $\liminf_{T\rightarrow\infty}\frac{\sup_{\nu\in\mathfrak{P}_{\mathrm{sg}}}\left(R_T(\nu)+(S_T(\nu))^\rho\right)}{\sqrt{T}}>c_1$ for some constant $c_1>0$. Consider an active learning policy. By \Cref{thm:lower-bound}, there exists an increasting sequence of $T_1,T_2,...$ with $\lim_{j\rightarrow\infty}T_j=\infty$, and a large constant $J>0$ such that for all $j\geq J$,
    \begin{align}\label{eq:RS_lowerbound}
        \mathcal R_{T_j}\cdot\mathcal S_{T_j}>CT_{j}^{\frac{3}{2}}.
    \end{align}
    If $\liminf_{j\rightarrow\infty}\mathcal R_{T_j}>c_2T_j^{\frac{3\rho}{2+2\rho}}$ for some constant $c_2>0$, then we are done. Otherwise, for any $\epsilon>0$, there must exist $j_1,j_2,...$ increasing to $\infty$, and a constant $L>0$ such that $\mathcal R_{T_{j_l}}<\epsilon T_{j_l}^{\frac{3\rho}{2+2\rho}}$ for all $l\geq L$. By \eqref{eq:RS_lowerbound}, this implies that $\mathcal S_{T_{j_l}}>\frac{C}{\epsilon}T_{j_l}^{\frac{3}{2+2\rho}}$ for all $l\geq L$. Equivalently, $\sup_{\nu\in\mathfrak{P}_{\mathrm{sg}}}\left(R_T(\nu)+(S_T(\nu))^\rho\right)\geq \frac{C}{\epsilon}T_{j_l}^{\frac{3\rho}{2+2\rho}}$. \qed
\end{proof}
\begin{proof}{Proof of \Cref{cor:minimax-sampling-instability}}
    Consider any minimax optimal algorithm that satisfies $\mathcal R_T\leq c_0\sqrt T$ for some $c_0>0$ over instance class $\mathfrak{P}_{\mathrm{sg}}$. Consider the instance where both arms are $\mathcal N(0,1)$. By \Cref{thm:minimax-worst-case-interval}, for $T$ sufficiently large, $\sqrt{\var (N_{1,T})}>c_2T$ for some $c_2>0$. 

Suppose the algorithm satisfies sampling stability. There must exist deterministic scalars $n_T\to\infty$ such that
\[
\frac{N_{1,T}}{n_T}\xrightarrow{p}1.
\]
First, we claim $n_T\le 2T$ for all large $T$. By convergence in probability, for all large $T$ we have
$P(N_{i,T}\ge n_T/2)\ge 1/2$. But $N_{i,T}\le T$, hence $n_T/2\le T$.

Now fix any $\delta\in(0,1)$. By convergence in probability, for all large $T$,
\[
P\big(|N_{i,T}-n_T|\ge \delta n_T\big)\le \delta.
\]
Using the bound $0\le N_{i,T}\le T$ and $n_T\le 2T$, we obtain
\begin{align*}
\mathbb E\big[(N_{i,T}-n_T)^2\big]
&\le (\delta n_T)^2 \;+\; (T+n_T)^2\,P\big(|N_{i,T}-n_T|\ge \delta n_T\big)\le(4\delta^2+9\delta)\,T^2.
\end{align*}
Finally, Jensen's inequality gives $(\mathbb E[N_{i,T}]-n_T)^2\le \mathbb E[(N_{i,T}-n_T)^2]$,
so
\[
\var(N_{i,T})
\le \mathbb E[(N_{i,T}-n_T)^2] + (\mathbb E[N_{i,T}]-n_T)^2
\le 2\,\mathbb E[(N_{i,T}-n_T)^2]
\le 2(4\delta^2+9\delta)T^2.
\]
Because $\delta>0$ is arbitrary, this implies $\var(N_{i,T})=o(T^2)$, a contradiction. \qed
\end{proof}

\end{document}